\documentclass{article} 
\usepackage{preprint_style,times}

\usepackage{amsmath,amsfonts,bm}

\def\eqref#1{equation~\ref{#1}}

\def\1{\bm{1}}

\DeclareMathAlphabet{\mathsfit}{\encodingdefault}{\sfdefault}{m}{sl}
\SetMathAlphabet{\mathsfit}{bold}{\encodingdefault}{\sfdefault}{bx}{n}

\usepackage{titletoc}
\usepackage{graphicx}
\usepackage{booktabs}
\usepackage{hyperref}
\usepackage{url}
\usepackage{enumitem}
\usepackage{algorithm}
\usepackage{algpseudocode}
\usepackage{amsthm}
\newtheorem{theorem}{Theorem}[section]
\newtheorem{proposition}[theorem]{Proposition}

\title{Optimal Design for Active Preference Learning with Biased LLM Judges}

\author{
Zhongman Du$^{1}$\quad Huiming Zhang$^{1}$\quad Haodong Zhu$^{1,2}$\quad Baochang Zhang$^{1,3}$\\[2pt]
{\normalfont $^{1}$Beihang University\qquad $^{2}$Zhongguancun Academy}\\[2pt]
{\normalfont $^{3}$Hangzhou Innovation Institute of Beihang University}
}

\begin{document}

\maketitle

\begin{abstract}
Learning from human preferences is central to large language model (LLM) alignment, but human preference annotation is costly. Active preference learning reduces this cost by selecting informative comparisons, and LLM judges can provide additional scalable feedback. However, the preferences of the judges may deviate from those of the target human population. Even after calibration on trusted reference data, active acquisition can shift the comparison distribution and expose residual judge bias. We therefore incorporate judge deviations into the acquisition design rather than relying on a separate calibration stage. Under joint estimation, comparisons that appear highly informative about the reward may also reflect judge bias and therefore provide less information about human preferences. To address this issue, we propose Nuisance-Adjusted Optimal Design (NAOD), a comparison-selection strategy that prioritizes policy-relevant target information after nuisance adjustment and uses the Frank--Wolfe algorithm for optimization. Theoretically, we establish a sharp conditional local asymptotic minimax lower bound on policy risk and construct an estimator that attains it. We further characterize the finite-sample cost of learning the nuisance representation and show that representation error can reverse an oracle design advantage. Finally, we validate these predictions experimentally and evaluate NAOD on Chatbot Arena data across 17 judges, 15 budget configurations, and 15 random cluster-level splits. NAOD reduces the mean regret of proxy policy by 29.1\% relative to a matched target-information design, outperforms existing methods, and improves human-preference prediction on held-out data.
\end{abstract}

\section{Introduction}

Preference feedback \citep{christiano2017deep} provides a standard way to align large language models (LLMs) with human preferences when desired behavior is difficult to specify directly \citep{ouyang2022training,rafailov2023direct}. Reward models are trained on pairwise response comparisons and then used to guide downstream policy optimization \citep{stiennon2020learning}. The resulting policy therefore depends not only on how the reward model is trained, but also on which comparisons are selected for feedback. When human annotation is costly or limited, active preference learning and experimental design make more efficient use of a fixed annotation budget by selecting informative comparisons \citep{pmlr-v235-muldrew24a,mukherjee2024optimal}. Recent methods further incorporate reward-model information \citep{lin2026activedpo} or downstream policy objectives \citep{pmlr-v267-feng25g} into their acquisition criteria.

A complementary approach to reducing annotation cost is to obtain preference feedback from LLM judges \citep{pmlr-v235-lee24t}. However, their judgments can differ systematically from the preferences of the target human population \citep{zheng2023judging}. A natural response is to calibrate the judge against trusted human data before relying on its feedback \citep{polo2025bridging}. However, calibration can deteriorate under distribution shift \citep{ovadia2019can,pmlr-v119-chan20a}, while active acquisition itself changes the comparison distribution by selecting which pairs receive feedback \citep{farquhar2021statistical}. Residual judge biases that cancel under the reference calibration distribution can therefore become unbalanced on the acquired distribution. Thus, pair selection determines both how much information is collected and which judge errors can influence the learned reward.

It is additionally difficult to model these judge deviations during estimation. The human reward and judge-specific deviation must now be estimated jointly and distinguished from one another. On some selected pairs, changes in the human reward and changes in the judge deviation can have nearly the same effect on label probabilities. A design that treats all apparent reward information as target information can therefore overstate the precision with which the human target can be estimated.

\begin{figure*}[t]
\centering
\includegraphics[width=\textwidth]{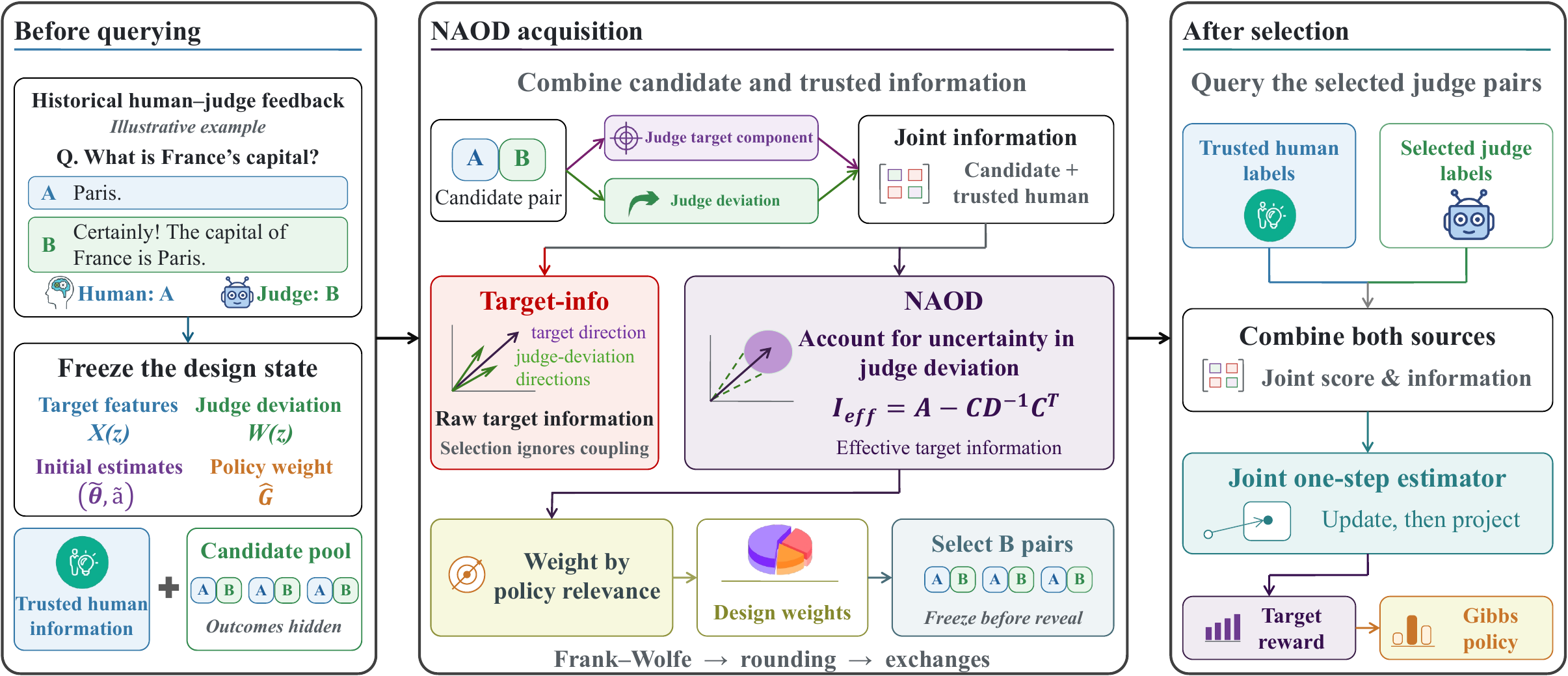}

\caption{\textbf{Overview of the NAOD acquisition and estimation pipeline.} Historical human--judge feedback is used to construct the frozen target and judge-deviation representations. Before current outcomes are observed, NAOD selects comparisons using nuisance-adjusted target information weighted by policy relevance. After selection, trusted human and judge feedback are combined in a joint estimator to update the target reward and induced policy.}
\label{fig:overview}
\vspace{-10pt}
\end{figure*}

For acquisition, what matters is not the raw information about reward parameters, but the information about the human target that remains after judge-specific deviations have been accounted for. Even this remaining information is not equally important across reward directions, because some reward errors substantially alter downstream policy decisions whereas others have little effect on the induced policy. Reward-model accuracy alone therefore need not reflect downstream policy performance \citep{pmlr-v202-gao23h,wen2025rethinking,frick2025evaluate}. We therefore propose Nuisance-Adjusted Optimal Design (NAOD), a design strategy that selects comparisons using nuisance-adjusted target information weighted by downstream policy relevance. Figure~\ref{fig:overview} provides an overview of the NAOD pipeline from pre-acquisition design to post-selection estimation.

A sharp conditional local asymptotic minimax lower bound for policy risk is established, together with an estimator that attains it. The resulting risk coincides with the nuisance-adjusted, policy-weighted objective optimized by NAOD. The same target–nuisance cross-information controls both residual exposure and the loss of target information. Further, when the judge-deviation representation is learned, a design-dependent risk term appears that can reverse oracle design rankings.


Empirically, controlled experiments validate the predicted policy-risk behavior and the design-ranking reversal. On Chatbot Arena, across 17 judges, 15 budget configurations, and 15 random cluster-level splits, NAOD reduces mean proxy policy regret by \(29.1\%\) relative to the matched Target-info design and improves held-out human-preference prediction. Its gains tend to be larger when target–nuisance coupling is stronger, consistent with the mechanism predicted by our theory.

Our contributions can be summarized as follows:
\begin{itemize}
    \item We formulate active preference learning with biased LLM judges as a joint target--nuisance design problem and characterize how acquisition exposes judge residuals.
    
    \item We introduce NAOD, establish its sharp local minimax interpretation, and characterize the additional risk from learning the judge-deviation representation.
    
    \item We develop a practical comparison-selection algorithm for finite candidate pools, validate our theoretical risk predictions in synthetic experiments, and evaluate NAOD on Chatbot Arena with held-out human evaluation.
\end{itemize}

\section{Related Work}
\label{sec-related-work}

\paragraph{Active preference acquisition and nuisance-aware design.}
Active preference learning reduces feedback cost by selecting informative comparisons. Prior work studies adaptive pairwise comparisons \citep{pmlr-v70-maystre17a}, preference-based reinforcement learning \citep{pmlr-v139-lee21i}, query-efficient learning from pairwise preferences \citep{wu2024randomization}, and subset selection for reliable LLM ranking \citep{liusie-etal-2024-efficient}. Information-based methods use Fisher information for reward modeling \citep{pmlr-v267-shen25c} or D-optimal design with randomized Frank--Wolfe optimization \citep{pmlr-v267-thekumparampil25a}, while other work incorporates downstream policy relevance or reward uncertainty into acquisition \citep{hu2024querypolicy,sun2025uncertainty}. When observations jointly inform target and nuisance parameters, acquisition must also account for nuisance uncertainty \citep{sloman2024bayesian}. Our work combines policy-aware preference acquisition with nuisance-aware experimental design under LLM-judge feedback.

\paragraph{LLM judges and calibration.}
LLM judges provide scalable evaluation but exhibit systematic biases, including position effects \citep{wang-etal-2024-large-language-models-fair}, stylistic preferences \citep{feuer2025style}, and broader recurring biases across human and model evaluators \citep{chen-etal-2024-humans,ye2025justice}. Related work also studies pairwise feedback from systematically biased evaluators \citep{tang2025biased}. Existing mitigation approaches characterize limits of debiasing with limited high-quality labels \citep{dorner2025limits}, fine-tune evaluators on debiased data \citep{park-etal-2024-offsetbias}, optimize prompts \citep{zhou-etal-2024-fairer}, calibrate pairwise prediction distributions \citep{li-etal-2025-calibraeval}, or use human calibration data to correct judge estimates and allocate calibration samples \citep{lee2026correctly}. Our focus is instead on how judge discrepancies should change which comparisons are acquired for human reward estimation.

\section{Problem Setting}
\label{sec-problem}

\paragraph{The Bradley-Terry model of human preferences.}
Let \(z=(s,y_A,y_B)\) denote an ordered comparison, where \(s\) is a prompt and \(y_A,y_B\) are two candidate responses. Let \(\phi(s,y)\in\mathbb{R}^d\) be a fixed \(d\)-dimensional feature map. We model the reward as linear in this feature representation, \(r_\theta(s,y)=\phi(s,y)^\top\theta\), where \(\theta\in\mathbb{R}^d\) is the reward parameter. We define the comparison feature \(X(z)=\phi(s,y_A)-\phi(s,y_B)\). Let \(Y\in\{0,1\}\) denote the human preference label, with \(Y=1\) indicating that \(y_A\) is preferred to \(y_B\). Under the Bradley-Terry model \citep{bradley1952rank}, the preference probability is determined by the reward difference between the two responses,
\begin{equation}
    p_\theta(z)
    =
    \mathbb{P}_\theta(Y=1\mid z)
    =
    \sigma\!\left(
        r_\theta(s,y_A)-r_\theta(s,y_B)
    \right)
    =
    \sigma\!\left(X(z)^\top\theta\right),
    \label{eq-human-preference}
\end{equation}
where \(\sigma(u)=(1+e^{-u})^{-1}\) is the sigmoid function. We write \(\theta^\star\) for the reward parameter of the target human population and \(p^\star=p_{\theta^\star}\) for its preference probability.

\paragraph{Performance metric.}
We evaluate the learned reward through the policy it induces. Let \(\mu_e\) denote the evaluation distribution over prompts. For each prompt \(s\), let \(\mathcal{A}(s)\) be a finite set of candidate responses and let \(\pi_{\mathrm{ref}}(\cdot\mid s)\) be a reference policy on \(\mathcal{A}(s)\) with full support. Given a regularization parameter \(\tau>0\), let \(\pi_\theta(\cdot\mid s)\) denote the optimizer of the KL-regularized reward objective \citep{azar2012dynamic,liu2024provably}. Its closed form is
\begin{equation}
    \pi_\theta(y\mid s)
    =
    \frac{
        \pi_{\mathrm{ref}}(y\mid s)
        \exp\!\left(r_\theta(s,y)/\tau\right)
    }{
        \sum_{y'\in\mathcal{A}(s)}
        \pi_{\mathrm{ref}}(y'\mid s)
        \exp\!\left(r_\theta(s,y')/\tau\right)
    }.
    \label{eq-induced-policy}
\end{equation}
The corresponding optimal regularized value is
\begin{equation}
    F(\theta)
    =
    \tau\,
    \mathbb{E}_{s\sim\mu_e}
    \log
    \sum_{y\in\mathcal{A}(s)}
    \pi_{\mathrm{ref}}(y\mid s)
    \exp\!\left(r_\theta(s,y)/\tau\right).
    \label{eq-policy-value}
\end{equation}
When the target reward parameter is \(\theta^\star\), deploying the policy induced by \(\theta\) incurs the loss
\begin{equation}
    D_F(\theta^\star,\theta)
    =
    F(\theta^\star)-F(\theta)
    -\nabla F(\theta)^\top(\theta^\star-\theta)
    =
    \tau\,
    \mathbb{E}_{s\sim\mu_e}
    \mathrm{KL}\!\left(
        \pi_\theta(\cdot\mid s)
        \,\|\, 
        \pi_{\theta^\star}(\cdot\mid s)
    \right).
    \label{eq-policy-loss}
\end{equation}
See Appendix~\ref{app-foundations-notation} for derivations.
Thus, our performance metric measures error through its effect on the induced policy.

\paragraph{LLM judge and active acquisition.}

For a fixed LLM judge and annotation protocol, let \(q(z)\in[0,1]\) denote its preference probability for comparison \(z\). Its preferences need not match those of the target human population. We define the judge discrepancy on the probability scale as \(\delta(z)=q(z)-p^\star(z)\). Let \(\nu\) denote a reference distribution over comparisons and let \(Z\sim\nu\). Conventional probability calibration with respect to the target human population requires \(\mathbb{E}_{\nu}[Y\mid q(Z)]=q(Z)\) \citep{pmlr-v70-guo17a,futami2024information,chidambaram2025reassessing}. Under the target preference model in \eqref{eq-human-preference}, this is equivalent to
\[
    \mathbb{E}_{\nu}[\delta(Z)\mid q(Z)] = 0.
\]
Thus, under \(\nu\), calibration constrains the conditional mean of the judge discrepancy given \(q(Z)\). However, active acquisition changes the comparison distribution. Let \(\xi\) denote the acquired comparison distribution and assume \(\xi\ll\nu\), so that the density ratio \(\omega(z)=d\xi/d\nu(z)\) is well defined.

\paragraph{Acquisition-dependent residual exposure.}

Probability calibration is defined under the reference distribution \(\nu\) and is not guaranteed to hold under the acquired distribution \(\xi\) \citep{pmlr-v108-park20b}. For reward estimation, the relevant issue is how judge discrepancies align with the reward-feature directions. Under the acquired distribution \(\xi\), define the residual contribution to the target score as \(b_\xi=\mathbb{E}_{\xi}[\delta(Z)X(Z)]\), and define \(b_\nu=\mathbb{E}_{\nu}[\delta(Z)X(Z)]\) analogously under the reference distribution. This residual-feature moment can be viewed as a multiaccuracy-style residual moment over linear reward-feature tests \citep{kim2019multiaccuracy,kern2024multiaccurate}. A change of measure gives
\begin{equation}
    b_\xi
    =
    b_\nu
    +
    \mathbb{E}_{\nu}
    \left[
        \{\omega(Z)-1\}\delta(Z)X(Z)
    \right].
    \label{eq-residual-exposure}
\end{equation}
The second term captures how acquisition reweights judge discrepancies across reward-feature directions. Calibration does not imply \(b_\nu=0\), and even when \(b_\nu=0\), acquisition can yield \(b_\xi\neq0\). In that case, treating judge feedback as target feedback shifts the population score away from zero at \(\theta^\star\), potentially moving the fitted reward away from the human target. Appendix~\ref{app-exposure} gives the corresponding parameter and policy-regret consequences.

Thus, active acquisition can expose judge discrepancies that are balanced under the reference distribution. This motivates accounting for judge deviations directly in the acquisition design rather than relying on reference calibration alone.

\section{Nuisance-Adjusted Optimal Design}
\label{sec-naod}

\subsection{Joint Target-Nuisance Model}
\label{sec-joint-model}

We model human and judge preferences jointly. Let \(W(z)\in\mathbb{R}^r\) be a fixed representation of judge deviations, where \(r\) is its dimension, and let \(a\in\mathbb{R}^r\) be the associated nuisance parameter. We model the fixed judge preference probability \(q(z)\) introduced in Section~\ref{sec-problem} using a Bradley--Terry model augmented with a structured judge-specific component \citep{movva2026whats},
\begin{equation}
    q_{\theta,a}(z)
    =
    \sigma\!\left(
        X(z)^\top\theta+W(z)^\top a
    \right).
    \label{eq-joint-judge-model}
\end{equation}
Human labels follow the preference model \(p_\theta\) in \eqref{eq-human-preference}, while \(W(z)^\top a\) represents the difference between judge and human log odds. For the theory in this section, we assume that \(q(z)=q_{\theta^\star,a^\star}(z)\) for some unknown nuisance parameter \(a^\star\). Thus, \(W(z)^\top a^\star\) parameterizes the judge-human discrepancy on the log-odds scale, while \(\delta(z)\) measures the corresponding discrepancy on the probability scale. Throughout this section, \(W\) is treated as fixed. Section~\ref{sec-representation-cost} studies the additional error when \(W\) is learned from historical human and judge feedback.

For the approximate-design formulation, we first consider a finite support of \(L\) comparison types \(z_1,\ldots,z_L\). Write \(X_i=X(z_i)\) and \(W_i=W(z_i)\). At type \(i\), we collect \(n_i\) judge labels, each indicating a preference for the first response with probability \(q_{\theta,a}(z_i)\). An independent trusted sample contains \(n_c\) human labels, whose preference probabilities are given by \(p_\theta\). Conditional on the comparison features, all current labels are independent Bernoulli observations. Let \(n=\sum_{i=1}^L n_i\) be the total number of judge labels. We represent the asymptotic allocation by \(\xi=(\xi_1,\ldots,\xi_L)\), where \(n_i/n\to\xi_i\), \(\xi_i\ge0\), and \(\sum_{i=1}^L\xi_i=1\). We also assume that \(n_c/n\to\kappa\) for some \(\kappa\in[0,\infty)\). We condition on the historical data used to construct \(W\) and select the allocation. These data are independent of the current labels.

\subsection{Effective Target Information}
\label{sec-effective-information}

Joint estimation must separate target-score variation from nuisance-score variation \citep{sloman2024bayesian,barker2001efficiency}. Let \(k=d+r\) be the joint parameter dimension. Define \(\gamma=(\theta^\top,a^\top)^\top\in\mathbb{R}^k\) and the corresponding feature \(v_i=(X_i^\top,W_i^\top)^\top\). Fix a local analysis center \(\gamma_0=(\theta_0^\top,a_0^\top)^\top\), where \(\theta_0\) and \(a_0\) are the target and nuisance components, respectively. The Bernoulli variance of the judge label at this center is \(t_i=q_{\theta_0,a_0}(z_i)\{1-q_{\theta_0,a_0}(z_i)\}\), and one judge label contributes Fisher information \(J_i=t_i v_i v_i^\top\). Let \(H_c\) denote the limiting Fisher information per trusted label at \(\theta_0\). Let \(X_c\) denote a random comparison feature in the trusted sample. Then \(H_c=\mathbb{E}[\sigma'(X_c^\top\theta_0)X_cX_c^\top]\). Let \(A_\kappa(\xi)\), \(D_\xi\), and \(C_\xi\) denote the target, nuisance, and cross-information blocks, respectively. The joint information normalized by the judge sample size is
\begin{equation}
    M_\kappa(\xi)
    =
    \kappa\operatorname{diag}(H_c,0_{r\times r})
    +
    \sum_{i=1}^L\xi_i J_i
    =
    \begin{pmatrix}
        A_\kappa(\xi) & C_\xi\\
        C_\xi^\top & D_\xi
    \end{pmatrix}.
    \label{eq-joint-information}
\end{equation}
Here, \(A_\kappa(\xi)\) is the target-information block, \(D_\xi\) is the nuisance-information block, and \(C_\xi\) is the target--nuisance cross-information block. Specifically, \(A_\kappa(\xi)=\kappa H_c+\sum_{i=1}^L\xi_i t_i X_iX_i^\top\), \(C_\xi=\sum_{i=1}^L\xi_i t_i X_iW_i^\top\), and \(D_\xi=\sum_{i=1}^L\xi_i t_i W_iW_i^\top\). When \(M_\kappa(\xi)\) is positive definite, the effective information for the human reward parameter is
\begin{equation}
    I_{\mathrm{eff},\kappa}(\xi)
    =
    A_\kappa(\xi)-C_\xi D_\xi^{-1}C_\xi^\top.
    \label{eq-effective-information}
\end{equation}
Estimating the nuisance parameter reduces the information available for \(\theta\) from \(A_\kappa(\xi)\) to the Schur complement \(I_{\mathrm{eff},\kappa}(\xi)\), with information loss \(C_\xi D_\xi^{-1}C_\xi^\top\) \citep{fewster2013information}. The inverse of \(I_{\mathrm{eff},\kappa}(\xi)\) is the target block of \(M_\kappa(\xi)^{-1}\).

The cross-information also connects this adjustment to the residual exposure in \eqref{eq-residual-exposure}. Holding the human parameter at \(\theta_0\), define \(b_\xi(a)=\sum_{i=1}^L\xi_i\{q_{\theta_0,a}(z_i)-p_{\theta_0}(z_i)\}X_i\). Let \(C_\xi(a)\) denote the cross-information evaluated at \((\theta_0,a)\), so that \(C_\xi(a_0)=C_\xi\). Differentiation gives the exact Jacobian identity
\begin{equation}
    \frac{\partial b_\xi(a)}{\partial a^\top}
    =
    C_\xi(a).
    \label{eq-exposure-coupling}
\end{equation}
The same cross-information governs the sensitivity of the exposed target score and the information loss from nuisance estimation. See Appendices~\ref{app-foundations-lan} and~\ref{app-foundations-efficient} for the statistical foundations of the joint and effective information, and Appendices~\ref{app-exposure} and~\ref{app-coupling} for the residual-exposure and target--nuisance coupling results.

\subsection{Policy-Weighted Design and Minimax Risk}
\label{sec-policy-design}

Effective information describes how precisely the human reward can be estimated after nuisance adjustment. The policy regret in \eqref{eq-policy-loss} determines which estimation errors matter. Define the policy curvature at \(\theta_0\) by
\[
    G_0
    =
    \nabla^2F(\theta_0)
    =
    \frac{1}{\tau}
    \mathbb{E}_{s\sim\mu_e}
    \operatorname{Cov}_{y\sim\pi_{\theta_0}(\cdot\mid s)}
    [\phi(s,y)].
\]
For \(h\in\mathbb{R}^d\), the local regret satisfies
\(D_F(\theta_0,\theta_0+h)=\tfrac12 h^\top G_0h+o(\|h\|^2)\) as \(h\to0\). Thus, \(G_0\) weights reward-parameter directions by their effect on the induced policy. For a reward estimator \(\widehat\theta_n\), we measure policy risk by \(\mathbb{E}[D_F(\theta^\star,\widehat\theta_n)]\). Let \(G_e=\operatorname{diag}(G_0,0_{r\times r})\), and let \(\operatorname{tr}\) denote the matrix trace. Combining policy curvature with effective target information gives
\begin{equation}
    \Phi_\kappa(\xi)
    =
    \frac12\operatorname{tr}\!\left(
        G_0 I_{\mathrm{eff},\kappa}(\xi)^{-1}
    \right)
    =
    \frac12\operatorname{tr}\!\left(
        G_e M_\kappa(\xi)^{-1}
    \right).
    \label{eq-policy-design-criterion}
\end{equation}
This is a weighted optimal design criterion for the target parameter \citep{pukelsheim2006optimal}. Let \(\Xi\) be a nonempty compact convex set of allocations on the \(L\) comparison types, and assume that \(M_\kappa(\xi)\) is uniformly positive definite over \(\Xi\). We define nuisance-adjusted optimal design (NAOD) by
\[
    \xi^\star
    \in
    \operatorname*{arg\,min}_{\xi\in\Xi}
    \Phi_\kappa(\xi).
\]
Since \(M_\kappa(\xi)\) is affine in \(\xi\), the objective is convex. To state its local minimax characterization, let \(\eta=(h^\top,\upsilon^\top)^\top\), where \(\upsilon\in\mathbb{R}^r\), and define \(\gamma_{n,\eta}=\gamma_0+\eta/\sqrt{n}\), with target component \(\theta_{n,h}=\theta_0+h/\sqrt{n}\). Write \(\mathbb{E}_\eta\) for expectation under this local model, conditional on the historical data.

\begin{theorem}[Conditional local policy risk]
\label{thm-local-policy-risk}
Fix \(\xi\) and condition on the historical data. Under the regularity conditions in Appendix~\ref{app-foundations-notation}, suppose \(M_\kappa(\xi)\succ0\). If \(G_0\succ0\), then
\begin{equation}
    \lim_{R\to\infty}
    \liminf_{n\to\infty}
    \inf_{T_n}
    \sup_{\|\eta\|\le R}
    n\,\mathbb{E}_\eta
    D_F(\theta_{n,h},T_n)
    \ge
    \Phi_\kappa(\xi),
    \label{eq-local-risk-lower-bound}
\end{equation}
where \(T_n\) ranges over reward estimators based on the current observations. Under the additional trusted-sample and initialization conditions in Appendix~\ref{app-minimax}, with \(\kappa>0\), the projected one-step estimator defined there attains this bound uniformly on every fixed local ball.
\end{theorem}

Thus, for the projected one-step estimator,
\(\mathbb{E}_\eta D_F(\theta_{n,h},\widehat\theta_n)
=\Phi_\kappa(\xi)/n+o(n^{-1})\)
uniformly over every fixed local ball. The case \(G_0\succeq0\) is handled by removing policy-null directions.

A matched Target-info design replaces \(I_{\mathrm{eff},\kappa}(\xi)\) by \(A_\kappa(\xi)\) in \eqref{eq-policy-design-criterion} while retaining the same final joint estimator. The two criteria can select different designs, and Proposition~\ref{prop-design-separation} shows that Target-info can incur strictly larger leading policy risk.

\subsection{Finite-Pool Comparison Selection}
\label{sec-finite-pool}

We now translate the approximate design into the selection of \(B\) distinct comparisons from a finite pool of \(N\) candidates. Let \(x_i\in[0,1]\) be the relaxed inclusion weight of candidate \(i\), and write \(x=(x_1,\ldots,x_N)\). Before querying the pool, we evaluate the information atoms and policy weight at a center fitted from historical feedback, yielding \(\widehat J_i\) and \(\widehat G_e\). Let \(\widehat I_c\) denote the corresponding total trusted information, embedded in the joint parameter space. The fitted information under \(x\) is \(\widehat{\mathcal I}(x)=\widehat I_c+\sum_{i=1}^N x_i\widehat J_i\).

Let \(\mathcal C\) partition the candidate indices into groups from which at most one comparison may be selected. Without group restrictions, each candidate forms its own group. Let \(S_0\) be a fixed feasible seed, possibly empty, used when needed to ensure positive definite fitted information. Its comparisons count toward the budget \(B\). The relaxed feasible set is
\[
    \mathcal X_B
    =
    \left\{
        x\in[0,1]^N
        \;\middle|\;
        \sum_{i=1}^N x_i=B,\quad
        x_i=1\ \text{for }i\in S_0,\quad
        \sum_{i\in c}x_i\le1\ \text{for }c\in\mathcal C
    \right\}.
\]
We solve
\begin{equation}
    \min_{x\in\mathcal X_B} f_B(x),
    \qquad
    f_B(x)
    =
    \frac12\operatorname{tr}\!\left[
        \widehat G_e\widehat{\mathcal I}(x)^{-1}
    \right].
    \label{eq-finite-pool-criterion}
\end{equation}
For \(\xi_i=x_i/B\) and \(\kappa=n_c/B\), this is the fitted finite-pool counterpart of \(\Phi_\kappa(\xi)/B\), preserving the total-information scaling of policy risk.

We solve the convex relaxation using the Frank--Wolfe algorithm \citep{pmlr-v28-jaggi13}, followed by feasible largest-weight rounding and improving exchanges. Algorithm~\ref{alg-naod} in Appendix~\ref{app-finite-pool-optimization} gives the complete acquisition procedure, and the same appendix provides a computable optimization certificate for the returned subset.

Because rounding need not preserve the relaxed allocation proportions, the statistical guarantee is stated in terms of the information of the selected comparisons themselves. For an asymptotic sequence of selected sets \(S_B\) containing \(B\) distinct comparisons selected before their outcomes are observed, let \(J_{Bi}\) denote the true information of candidate \(i\) at \(\gamma_0\), and let \(I^{\mathrm{exp}}_{c,B}\) denote the total expected trusted information. Define the normalized information of the selected set by \(M_B=B^{-1}\!\left(I^{\mathrm{exp}}_{c,B}+\sum_{i\in S_B}J_{Bi}\right)\). Let \(\widehat\theta_B\) be the target component of the projected one-step estimator based on the trusted sample and the selected judge labels.

\begin{theorem}[Policy risk for distinct comparisons]
\label{thm-distinct-comparisons}
Condition on the historical data, candidate features, and selected set. Under the distinct-comparison regularity conditions in Appendix~\ref{app-distinct-comparisons}, suppose \(M_B\) is uniformly positive definite. Under the local model \(\gamma_0+\eta/\sqrt B\),
\begin{equation}
    B\,\mathbb{E}_\eta
    D_F\!\left(
        \theta_0+\frac{h}{\sqrt B},
        \widehat\theta_B
    \right)
    -
    \frac12\operatorname{tr}(G_eM_B^{-1})
    \longrightarrow0
    \label{eq-distinct-policy-risk}
\end{equation}
uniformly over bounded local parameters. If \(M_B\) converges to a positive definite limit, the corresponding conditional local minimax lower bound also holds.
\end{theorem}

The theorem evaluates the information of the realized subset and therefore retains the effect of rounding. Appendix~\ref{app-distinct-comparisons} gives the proof, while Appendix~\ref{app-computation} gives the plug-in analysis and computational details.

\section{Learning the Judge-Deviation Representation}
\label{sec-representation-learning}

In this section, we study the additional error introduced when the judge-deviation representation is learned from historical feedback. This error contributes a design-dependent term to policy risk and can reverse the ordering of acquisition rules.

\subsection{Finite-Sample Representation Cost}
\label{sec-representation-cost}

We use the same finite support \(z_1,\ldots,z_L\) and evaluate policy risk at the human target \(\theta_0=\theta^\star\). Let \(m\) denote the number of independent historical units used to learn the judge-deviation representation. Write \(\widehat W_m\in\mathbb R^{L\times r}\) for the learned representation, with row \(i\) given by \(\widehat W_{m,i}^\top\). Let \(a_{0m}\in\mathbb R^r\) be the nuisance component of the local center \(\gamma_{0m}=(\theta_0^\top,a_{0m}^\top)^\top\). The nuisance parameter \(a\) is re-estimated from the current data. Write \(\operatorname{logit}(u)=\log\{u/(1-u)\}\) for \(u\in(0,1)\), and define the remaining logit error \(e_m=(e_{m,1},\ldots,e_{m,L})^\top\) by
\begin{equation}
e_{m,i}
=
\operatorname{logit}q(z_i)
-
X_i^\top\theta_0
-
\widehat W_{m,i}^\top a_{0m},
\qquad i=1,\ldots,L.
\label{eq-representation-residual}
\end{equation}

Let \(\xi_{m,i}=n_i/n\) denote the judge allocation proportion for type \(i\). Define \(v_{m,i}=(X_i^\top,\widehat W_{m,i}^\top)^\top\) and \(t_{m,i}=\sigma'(v_{m,i}^\top\gamma_{0m})\). Let \(M_m\) denote the normalized joint information at \(\gamma_{0m}\), formed as in \eqref{eq-joint-information} from the learned representation, the actual judge allocation, and the expected trusted information. Set \(\Phi_m=\tfrac12\operatorname{tr}(G_eM_m^{-1})\), and let \(P_\theta=[I_d\;0_{d\times r}]\) select the target coordinates. For \(e=(e_1,\ldots,e_L)^\top\in\mathbb R^L\), define
\begin{equation}
    B_m e
    =
    P_\theta M_m^{-1}
    \sum_{i=1}^L
    \xi_{m,i}t_{m,i}v_{m,i}e_i,
    \qquad
    R_m
    =
    B_m^\top G_0B_m.
    \label{eq-representation-influence}
\end{equation}
The first-order effect of \(e_m\) on the target parameter is \(B_me_m\). Let \(\widehat\theta_{n,m}\) be the target component of the projected one-step estimator based on the learned representation and the current observations. Expectations below are over both the historical and current data.

\begin{theorem}[Policy risk with a learned representation]
\label{thm-representation-cost}
Under the regularity conditions in Appendix~\ref{app-representation-finite}, suppose \(n_c/n\to\kappa>0\) and the smallest eigenvalue of \(M_m\) is bounded away from zero. Let \(m,n\to\infty\) with \(n/m\to\lambda_P\in[0,\infty)\). Suppose \((\Phi_m,R_m)\to(\Phi_0,R_0)\) in probability for deterministic limits, and \(\sqrt m\,e_m\) converges in distribution to a random vector \(Z_P\) with mean \(b_P\) and covariance \(\Sigma_P\). Then
\begin{equation}
    \mathbb E\!\left[D_F(\theta_0,\widehat\theta_{n,m})\right]
    =
    \frac{\Phi_0}{n}
    +
    \frac{C_P}{m}
    +
    o(n^{-1}+m^{-1}),
    \qquad
    C_P
    =
    \frac12
    \left\{
        \operatorname{tr}(R_0\Sigma_P)
        +
        b_P^\top R_0b_P
    \right\}.
    \label{eq-representation-risk}
\end{equation}
\end{theorem}

The first term is the oracle sampling risk and the second is the representation-learning cost, which depends on the direction of the representation error through \(R_0\) and includes both covariance and bias contributions. Proofs and additional representation-learning results are given in Appendix~\ref{app-representation-learning}.

\subsection{Design Ranking Reversals}
\label{sec-design-reversals}

The representation-cost term in \eqref{eq-representation-risk} can reverse the ordering induced by the oracle sampling risk. We compare NAOD with the matched Target-info design under the same learned representation, label budgets, and final estimator. Let \(\mathcal L_{\mathrm N}(n,m)\) and \(\mathcal L_{\mathrm T}(n,m)\) denote their expected policy risks. Write \(\Phi_{\mathrm N},\Phi_{\mathrm T}\) for their limiting oracle risk constants and \(C_{P,\mathrm N},C_{P,\mathrm T}\) for their representation-cost coefficients. Applying Theorem~\ref{thm-representation-cost} to the two designs gives
\begin{equation}
    n\{\mathcal L_{\mathrm N}(n,m)-\mathcal L_{\mathrm T}(n,m)\}
    \longrightarrow
    \Phi_{\mathrm N}-\Phi_{\mathrm T}
    +
    \lambda_P(C_{P,\mathrm N}-C_{P,\mathrm T}).
    \label{eq-design-reversal}
\end{equation}
If \(\Phi_{\mathrm N}<\Phi_{\mathrm T}\), the leading-order ordering reverses whenever
\(\lambda_P(C_{P,\mathrm N}-C_{P,\mathrm T})>\Phi_{\mathrm T}-\Phi_{\mathrm N}\).
Thus, NAOD can have lower oracle sampling risk while being more sensitive to representation error.

When \(n/m\to0\), the representation-cost term vanishes and the oracle ordering is recovered. When \(n/m\to\lambda_P>0\), it can contribute at leading order. Thus, NAOD optimizes the fitted design criterion rather than the full two-stage risk.

\section{Experiments}
\label{sec-experiments}

\subsection{Synthetic Experiments}
\label{sec-controlled-experiments}

\paragraph{Effective information predicts policy risk.}
We first consider a scalar three-type Bernoulli construction with a constant nuisance component and equal trusted and judge sample sizes, \(n_c=n\). NAOD, Target-info, and Random use the same correctly specified joint model and projected one-step estimator and differ only in the acquisition design. Each setting is evaluated over 6,000 independent replications. Across \(n\in\{240,960,3840,15360\}\), the empirical scaled risks in Figure~\ref{fig-controlled}(a) closely track their corresponding theoretical information constants, supporting the effective-information characterization of leading policy risk.

\begin{figure}[htbp]
    \centering
    \includegraphics[width=\linewidth]{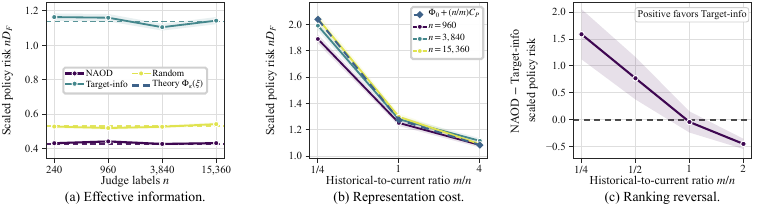}
    \vspace{-7mm}
    \caption{\textbf{Controlled tests of the theoretical predictions.} Dashed references in panels (a) and (b) denote the corresponding theoretical predictions. Panel (c) reports the NAOD-minus-Target-info scaled-risk difference.}
    \label{fig-controlled}
\end{figure}

 \vspace{-1mm}
 
\paragraph{Learning the representation adds the predicted risk term.}
We next use a scalar two-type Bernoulli construction with \(n_c=n\), in which the first-order cost of learning the judge-deviation representation is explicit. For this construction, \(\Phi_0=50/49\) and \(C_P=25/98\). Figure~\ref{fig-controlled}(b) tests the risk expansion in \eqref{eq-representation-risk} across \(n\in\{960,3840,15360\}\) and \(m/n\in\{1/4,1,4\}\). The empirical scaled risks closely follow the theoretical prediction \(\Phi_0+(n/m)C_P\), with the largest finite-sample deviation occurring at the smallest representation-learning sample size.

\paragraph{Representation error can reverse design rankings.}
We then learn the representation in the three-type experiment and compare NAOD with Target-info under otherwise identical conditions, with \(n=n_c=960\). Figure~\ref{fig-controlled}(c) shows the paired NAOD-minus-Target-info scaled-risk difference as \(m/n\) varies. With limited historical data, the difference is positive and Target-info has lower risk. As the historical sample grows, the gap shrinks, becomes indistinguishable from zero around \(m/n=1\), and turns negative by \(m/n=2\), where NAOD has lower risk. This confirms the design-ranking reversal mechanism in Section~\ref{sec-design-reversals}. Full constructions, finite-sample diagnostics, and additional analyses are reported in Appendices~\ref{app-controlled-mechanisms} and~\ref{app-controlled-additional}.

\subsection{Chatbot Arena Evaluation}
\label{sec-arena-experiments}

\paragraph{Experimental setup.}
We evaluate finite-pool acquisition on a frozen Chatbot Arena LLM-judge archive \citep{pmlr-v235-chiang24b} with \(49{,}635\) usable comparisons in \(42{,}875\) connected text clusters and 17 judges. Current budgets are \(H\in\{32,64,128,256,512\}\) human labels and \(B\in\{256,512,1024\}\) judge records, yielding 15 budget combinations. Each combination is evaluated over 15 cluster-level splits. Within each split, representation learning, acquisition, and evaluation use disjoint clusters. At most one comparison is selected from each candidate cluster.

All acquisition methods differ only in the acquisition rule. Frozen Qwen3 embeddings \citep{qwen3embedding} are used to construct judge-specific two-dimensional target and nuisance representations, with the full procedure given in Appendix~\ref{app-arena-protocol}. Representation learning uses \(8{,}575\) shared upstream human--judge comparisons, and initialization uses 128 historical comparisons. These shared resources are fixed across acquisition rules and are not counted in the current \((H,B)\) budgets. We compare NAOD against five acquisition baselines, Random, Entropy, D-opt, PA D-opt \citep{pmlr-v267-shen25c}, and Target-info. Initial and Human-only serve as reference baselines from the common initialization. Judge feedback uses the archived A/B/tie probabilities \citep{sun-etal-2025-skillaggregation} through the soft target \(p_A+\tfrac12p_T\), where \(p_A\) and \(p_T\) are the archived probabilities of an A win and a tie, respectively. For these bounded soft targets, Appendix~\ref{app-foundations-soft} gives the corresponding sandwich-risk characterization. Under mean correctness, Theorem~\ref{thm:app-soft-least-favorable} shows that the population curvature criterion underlying NAOD equals the worst-case leading sandwich policy-risk coefficient over bounded soft-feedback laws with the same conditional means.
 
\paragraph{Evaluation.}
The primary endpoint is proxy policy regret relative to a human-preference reference fitted only from upstream data. We additionally report held-out human cross-entropy and tie-aware human choice accuracy. Test losses are first averaged within connected clusters and then equally across clusters. Within each split, judges and budget combinations are macro-averaged. Paired uncertainty is reported using two-sided 95\% \(t\)-intervals across the 15 split-level differences.

\paragraph{Results.}
NAOD achieves the lowest mean proxy policy regret and human cross-entropy in Table~\ref{tab:arena-main}, together with the highest mean human choice accuracy. Relative to the matched Target-info design, NAOD reduces mean proxy policy regret by \(29.11\%\), with 14 of 15 split averages favoring NAOD. Relative to D-opt, the reduction is \(18.87\%\), with 12 of 15 split averages favoring NAOD.

\begin{table}[htbp]
\centering
\small
\setlength{\tabcolsep}{5pt}
\renewcommand{\arraystretch}{1.3}
\caption{\textbf{Overall Chatbot Arena results.} Means over 17 judges, 15 budget regimes, and 15 random cluster-level splits. Regret is reported in \(10^{-3}\) units, CE in nats per pair, and accuracy in percent. The last two columns report NAOD's relative regret reduction and paired split wins against each row. Bold denotes the best mean.}
\label{tab:arena-main}
\vspace{1.5mm}
\begin{tabular*}{\linewidth}{@{\extracolsep{\fill}}lccccc}
\hline
Method & Proxy regret $\downarrow$ & Human CE $\downarrow$ & Acc. $\uparrow$ & Drop (\%) & Wins \\
\hline
Initial     & 8.86  & 0.6752 & 57.82 & 77.53 & 13/15 \\
Human-only  & 18.25 & 0.6829 & 56.84 & 89.09 & 15/15 \\
Random      & 4.88  & 0.6727 & 58.22 & 59.18 & 15/15 \\
Entropy     & 13.42 & 0.6784 & 57.10 & 85.17 & 15/15 \\
D-opt       & 2.45  & 0.6703 & 58.19 & 18.87 & 12/15 \\
PA D-opt    & 2.46  & 0.6703 & 58.19 & 18.93 & 12/15 \\
Target-info & 2.81  & 0.6705 & 58.21 & 29.11 & 14/15 \\
\hline
\textbf{NAOD} & \textbf{1.99} & \textbf{0.6699} & \textbf{58.27} & --- & --- \\
\hline
\end{tabular*}
\end{table}

\paragraph{Budget dependence.}
NAOD reduces mean proxy policy regret relative to Target-info in all 15 budget settings. At a fixed human budget \(H\), the relative advantage consistently narrows as the judge budget \(B\) increases. Moreover, settings with the same ratio \(H/B\) can exhibit substantially different gains, showing that absolute budget size matters in addition to the human-to-judge budget ratio. The complete \(5\times3\) budget grid is reported in Figure~\ref{fig-arena-budgets} in Appendix~\ref{app-arena-results}.

\paragraph{Target--nuisance coupling.}
The acquisition gain varies across judges and tends to be larger when Target-info induces stronger target--nuisance coupling. This pattern is consistent with the theoretical prediction that nuisance-aware acquisition is most valuable when target and nuisance information are strongly coupled. The corresponding judge-level analysis is reported in Figure~\ref{fig-arena-coupling} in Appendix~\ref{app-arena-results}.

\section{Conclusion}
\label{sec-conclusion}

We studied active preference learning with LLM-judge feedback that may systematically deviate from target human preferences. NAOD prioritizes policy-relevant target information after nuisance adjustment. The design criterion has a sharp local minimax characterization, while controlled experiments validate the predicted risk behavior and Arena results show that gains tend to be larger under stronger
target--nuisance coupling. Representation error can nevertheless change the preferred acquisition design, motivating future work on jointly learning judge-deviation representations and selecting comparisons.

\subsection*{AI use statement}
Generative AI tools were used to assist with language polishing and to check the algebraic consistency of author-derived mathematical derivations. All AI-assisted suggestions and checks were reviewed and independently verified by the authors. The authors take full responsibility for the final content of the paper.

\subsection*{Ethics statement}
This work uses archived Chatbot Arena preference data and collects no new human-subject data. The available preference feedback and LLM-judge outputs may reflect population- and model-specific biases and may not be universally representative. Preference models and acquisition strategies based on such feedback should therefore be evaluated with respect to the intended target population and deployment setting.

\subsection*{Reproducibility statement}
We provide the code required to reproduce the reported experimental results in the supplementary material. The main text and appendices provide the statistical model, acquisition and estimation procedures, optimization details, and experimental protocols required for reproduction, together with the assumptions and proofs underlying the theoretical results.

\bibliography{iclr2027_conference}
\bibliographystyle{iclr2027_conference}

\clearpage

\appendix
\startcontents[appendix]
\section*{Appendix}
\section{Technical Foundations}
\label{app-foundations}

This appendix collects the common statistical and policy-loss facts used throughout the analysis. We specify the conditional experiment and policy loss, establish the local likelihood expansion, characterize efficient target information and policy-null directions, and finally separate likelihood curvature from score covariance for soft feedback.

\subsection{Notation, Regularity, and Policy Loss}
\label{app-foundations-notation}

\paragraph{Conditioning and notation.}
Let $\mathcal F_P$ denote the historical information available before the current outcomes are observed, including the frozen nuisance representation, comparison features, allocation, and numerical settings. Unless explicitly stated otherwise, probabilities and expectations in the fixed-representation theory are conditional on $\mathcal F_P$. Write $\gamma=(\theta^\top,a^\top)^\top\in\mathbb R^k$, with $k=d+r$, and let $P_\theta=[I_d\;\;0]$ select the target coordinates. A judge observation of comparison type $i$ has augmented feature $v_i=(X_i^\top,W_i^\top)^\top$, while a trusted observation with feature $X_{cj}$ has padded feature $v_{cj}=(X_{cj}^\top,0^\top)^\top$. There are $n_i$ judge observations of type $i$, $n=\sum_{i=1}^L n_i$ judge observations in total, and $n_c$ trusted observations. For bounded $\eta=(h^\top,\upsilon^\top)^\top$, the local parameter is $\gamma_{n,\eta}=\gamma_0+\eta/\sqrt n$, with target component $\theta_{n,h}=\theta_0+h/\sqrt n$.

\paragraph{Regularity.}
For the fixed-support theory, $L,d,r$ are fixed and all feature vectors are uniformly bounded. Conditional on their features, current judge and trusted outcomes are mutually independent Bernoulli variables under the joint model in \eqref{eq-joint-judge-model}. A compact convex product set $\mathcal K=\mathcal K_\theta\times\mathcal K_a$ contains $\gamma_0$ in its interior, with a bounded open neighborhood available for derivative expansions. The proportions satisfy $n_i/n\to\xi_i$ and $n_c/n\to\kappa<\infty$. Uniformly over bounded root-$n$ neighborhoods, the normalized expected information converges to the positive-definite matrix $M_\kappa(\xi)$ in \eqref{eq-joint-information}. Trusted covariates may be bounded deterministic arrays with the stated information limit or independent bounded random covariates independent of the judge block. The policy value $F$ is twice continuously differentiable near the target region, with bounded uniformly continuous Hessian. For the Gibbs policy considered here, boundedness follows directly from the covariance representation below, while uniform continuity follows on the compact target neighborhood. Bounded features and the compact parameter region also imply a common interior probability bound $q(1-q)\ge t_{\min}>0$. Additional conditions required by the concrete preliminary estimator are stated in Appendix~\ref{app-minimax}.

\paragraph{Gibbs policy and directed policy loss.}
For the Gibbs policy in \eqref{eq-induced-policy} and value function in \eqref{eq-policy-value}, substituting the log probability ratio into the regularized policy objective gives
\[
V_\theta(\pi)=F(\theta)-\tau\,\mathbb E_{\mu_e}\mathrm{KL}\!\left(\pi(\cdot\mid s)\,\|\,\pi_\theta(\cdot\mid s)\right).
\]
Thus $\pi_\theta$ is the unique optimizer. Differentiating the log partition function yields
\[
\nabla F(\theta)=\mathbb E_{s\sim\mu_e}\mathbb E_{y\sim\pi_\theta(\cdot\mid s)}[\phi(s,y)],\qquad \nabla^2F(\theta)=\frac{1}{\tau}\mathbb E_{s\sim\mu_e}\operatorname{Cov}_{y\sim\pi_\theta(\cdot\mid s)}[\phi(s,y)].
\]
In particular, $F$ is convex. If the policy optimized at parameter $y$ is evaluated under reward parameter $x$, linearity of the reward gives $V_x(\pi_y)=F(y)+\nabla F(y)^\top(x-y)$, and therefore
\[
F(x)-V_x(\pi_y)=D_F(x,y)=\tau\,\mathbb E_{\mu_e}\mathrm{KL}\!\left(\pi_y(\cdot\mid s)\,\|\,\pi_x(\cdot\mid s)\right).
\]
Hence the first Bregman argument is the reward under which performance is evaluated, whereas the second determines the deployed policy. At the local center,
\begin{equation}
D_F(\theta_0,\theta_0+h)=\frac12 h^\top G_0h+o(\|h\|^2),\qquad G_0=\nabla^2F(\theta_0).
\label{eq-app-local-policy-loss}
\end{equation}
This gives the policy weighting used by NAOD.

\subsection{Scores, Information, and Local Asymptotic Normality}
\label{app-foundations-lan}

Let $\operatorname{sp}(u)=\log(1+e^u)$. For judge labels $Y_{i\ell}$ and trusted labels $H_j$, the total negative log likelihood is
\[
Q_n(\gamma)=\sum_{i=1}^L\sum_{\ell=1}^{n_i}\left\{\operatorname{sp}(v_i^\top\gamma)-Y_{i\ell}v_i^\top\gamma\right\}+\sum_{j=1}^{n_c}\left\{\operatorname{sp}(X_{cj}^\top\theta)-H_jX_{cj}^\top\theta\right\}.
\]
Define the total score and observed information by $U_n(\gamma)=-\nabla Q_n(\gamma)$ and $I_n(\gamma)=\nabla^2Q_n(\gamma)$. A judge observation with probability $q_i(\gamma)=\sigma(v_i^\top\gamma)$ contributes $(Y_{i\ell}-q_i(\gamma))v_i$ to the score and $q_i(\gamma)\{1-q_i(\gamma)\}v_iv_i^\top$ to the information. Trusted observations have the same form with padded features. Hence $Q_n$ is globally convex.

At the correctly specified Bernoulli center,
\[
M_n:=\frac1n\mathbb E I_n(\gamma_0)=\frac1n\operatorname{Cov}\{U_n(\gamma_0)\}\longrightarrow M_\kappa(\xi).
\]
The equality between score covariance and curvature is specific to the correctly specified Bernoulli experiment. Appendix~\ref{app-foundations-soft} treats soft feedback separately. Let $\Delta_n=U_n(\gamma_0)/\sqrt n$. Bounded independent score summands in fixed dimension give $\sup_n\mathbb E\|\Delta_n\|^4<\infty$, while the normalized Hessian differs from its expectation by $O_{L^4}(n^{-1/2})$ when current trusted covariates are random. For fixed covariate arrays, it is deterministic conditional on the features.

\begin{proposition}[Conditional LAN]
\label{prop:app-lan}
Under the regularity conditions of Appendix~\ref{app-foundations-notation},
\[
\Delta_n\Rightarrow N\!\left(0,M_\kappa(\xi)\right),
\]
and, writing $\ell_n=-Q_n$, for every fixed $R<\infty$,
\[
\ell_n\!\left(\gamma_0+\frac{\eta}{\sqrt n}\right)-\ell_n(\gamma_0)=\eta^\top\Delta_n-\frac12\eta^\top M_\kappa(\xi)\eta+o_P(1)
\]
uniformly over $\|\eta\|\le R$. The same derivative and moment bounds hold uniformly along bounded local truth sequences.
\end{proposition}

\begin{proof}
Every scalar projection of $\Delta_n$ is a triangular array of independent centered bounded terms whose largest summand is $O(n^{-1/2})$ and whose total variance converges to the corresponding quadratic form of $M_\kappa(\xi)$. The triangular-array central limit theorem and the Cram\'er--Wold device give the Gaussian limit. A third-order Taylor expansion of the log likelihood around $\gamma_0$ gives the displayed quadratic likelihood ratio. Bounded third derivatives make the remainder $O(n^{-1/2})$ uniformly on each fixed local ball, while normalized information converges uniformly to $M_\kappa(\xi)$.
\end{proof}

\subsection{Efficient Target Information and Policy-Null Directions}
\label{app-foundations-efficient}

Suppress $(\kappa,\xi)$ and partition the correctly specified Bernoulli information as
\[
M=\begin{pmatrix}A&C\\C^\top&D\end{pmatrix}\succ0.
\]
Let $(S_\theta,S_a)$ denote the target and nuisance components of the limiting LAN score, whose covariance is $M$. Since $D\succ0$, define $S_{\rm eff}=S_\theta-CD^{-1}S_a$. Then
\[
\operatorname{Cov}(S_{\rm eff},S_a)=0,\qquad \operatorname{Cov}(S_{\rm eff})=A-CD^{-1}C^\top=I_{\rm eff}.
\]
Thus the Schur complement removes exactly the target-score component linearly reproducible by nuisance-score variation. Equivalently, for every target direction $h$,
\[
h^\top I_{\rm eff}h=\min_{b\in\mathbb R^r}\left\{h^\top Ah-2h^\top Cb+b^\top Db\right\},
\]
with minimizer $b=D^{-1}C^\top h$. Block elimination gives
\begin{equation}
P_\theta M^{-1}P_\theta^\top=I_{\rm eff}^{-1}.
\label{eq-app-block-inverse}
\end{equation}
Finally, with $K=A^{-1/2}CD^{-1/2}$,
\begin{equation}
A^{-1/2}I_{\rm eff}A^{-1/2}=I-KK^\top.
\label{eq-app-whitened-efficient}
\end{equation}
The latter is the whitening identity used by the target--nuisance coupling analysis.

At the saturation extreme, if the judge target-score directions lie entirely in the nuisance-score span, judge feedback contributes no additional effective information for the human target. In particular, if $W_i=X_i$ with unrestricted nuisance coefficients, the judge likelihood identifies only $\theta+a$, and $I_{\rm eff}=\kappa H_c$.

\paragraph{Policy-null directions.}
Theorem~\ref{thm-local-policy-risk} states the lower bound directly for $G_0\succ0$. For the Gibbs policy, the reduction required when $G_0\succeq0$ is exact.

\begin{proposition}[Exact policy-null quotient]
\label{prop:app-policy-null}
Let $\mathcal N=\ker\nabla^2F(\theta_0)$. For the finite-response Gibbs policy with full-support reference policy, $\mathcal N$ is the same at every finite parameter. For every $h\in\mathcal N$ and finite $t$,
\[
\pi_{\theta+th}(\cdot\mid s)=\pi_\theta(\cdot\mid s)\qquad\text{for $\mu_e$-almost every }s.
\]
In orthogonal coordinates $\theta=(\beta,\zeta)\in\mathcal N^\perp\oplus\mathcal N$, the value decomposes as $F(\beta,\zeta)=f(\beta)+c^\top\zeta$, so
\[
D_F\bigl((\beta,\zeta),(\beta',\zeta')\bigr)=D_f(\beta,\beta').
\]
Unless the policy loss is identically zero, $\nabla^2f(\beta_0)\succ0$.
\end{proposition}

\begin{proof}
For $h\in\mathcal N$, the covariance representation of the policy Hessian gives
\[
0=h^\top\nabla^2F(\theta_0)h=\frac1\tau\mathbb E_{\mu_e}\operatorname{Var}_{\pi_{\theta_0}(\cdot\mid s)}\{h^\top\phi(s,y)\}.
\]
Because the Gibbs policy has full support, $h^\top\phi(s,y)$ must be constant over responses for almost every prompt. Adding $th$ therefore multiplies every unnormalized Gibbs weight at that prompt by the same factor, leaving the normalized policy unchanged and making $F$ affine along $h$. The same response-constant characterization holds at every finite parameter, so the null space is parameter-independent. The affine component cancels from the Bregman divergence, giving the quotient representation.
\end{proof}

Policy-null target coordinates may therefore be treated as additional nuisance coordinates. Applying the positive-curvature argument on $\mathcal N^\perp$ yields the same policy-weighted trace criterion in the original coordinates.

\subsection{Soft Feedback and Sandwich Policy Risk}
\label{app-foundations-soft}

For binary evaluation comparisons with equal reference weights and $\tau=1$, let $F_e(\theta)=\mathbb E_e\operatorname{sp}(X^\top\theta)$ up to an affine term. Under the correctly specified human model,
\[
\operatorname{CE}(\theta)-\operatorname{CE}(\theta^\star)=D_{F_e}(\theta,\theta^\star),\qquad
\operatorname{Regret}(\pi_\theta;\theta^\star)=D_{F_e}(\theta^\star,\theta).
\]
The two losses therefore have opposite Bregman orientations, although both have the same local quadratic expansion
\[
\frac12(\theta-\theta^\star)^\top\nabla^2F_e(\theta^\star)(\theta-\theta^\star)
+o(\|\theta-\theta^\star\|^2).
\]
This distinction is used in Appendix~\ref{app-arena-analysis}, where proxy policy regret and held-out human cross-entropy are separate endpoints.

Now allow conditionally independent responses $Y_{nj}\in[0,1]$. The logistic loss remains $\operatorname{sp}(v_{nj}^\top\gamma)-Y_{nj}v_{nj}^\top\gamma$, with
\[
U_n(\gamma)=\sum_j v_{nj}\{Y_{nj}-q_{nj}(\gamma)\},\qquad
I_n(\gamma)=\sum_j q_{nj}(\gamma)\{1-q_{nj}(\gamma)\}v_{nj}v_{nj}^\top,
\]
where $q_{nj}(\gamma)=\sigma(v_{nj}^\top\gamma)$. Hence soft and Bernoulli responses share the same logistic curvature, while their score covariances need not agree.

\begin{proposition}[Soft-label sandwich policy risk]
\label{prop:app-soft-sandwich}
Suppose interior population roots $\gamma_n^\dagger=(\theta_n^{\dagger\top},a_n^{\dagger\top})^\top\to\gamma^\dagger$ satisfy $\mathbb E U_n(\gamma_n^\dagger)=0$, and
\[
M_n:=\frac1n\mathbb E I_n(\gamma_n^\dagger)\to M\succ0,\qquad
\Omega_n:=\frac1n\operatorname{Cov}\{U_n(\gamma_n^\dagger)\}\to\Omega\succeq0.
\]
Assume a regular estimator satisfies
\[
\sqrt n(\widehat\gamma_n-\gamma_n^\dagger)
=
M^{-1}\frac{U_n(\gamma_n^\dagger)}{\sqrt n}
+o_{L^2}(1).
\]
Partition $M=\left(\begin{smallmatrix}A&C\\C^\top&D\end{smallmatrix}\right)$, set $I_{\rm eff}=A-CD^{-1}C^\top$, $L=[I_d\;\;-CD^{-1}]$, and $\Omega_{\rm eff}=L\Omega L^\top$. Then, with $G^\dagger=\nabla^2F(\theta^\dagger)$,
\[
n\,\mathbb E D_F(\theta_n^\dagger,\widehat\theta_n)
\longrightarrow
\Psi_{\rm soft}
:=
\frac12\operatorname{tr}\!\left(
G^\dagger I_{\rm eff}^{-1}
\Omega_{\rm eff}
I_{\rm eff}^{-1}
\right).
\]
If every row is additionally mean-correct, $\mathbb E[Y_{nj}\mid v_{nj}]=q_{nj}(\gamma_n^\dagger)$, then
\[
\Omega\preceq M,\qquad
\Omega_{\rm eff}\preceq I_{\rm eff},\qquad
\Psi_{\rm soft}
\le
\frac12\operatorname{tr}\!\left(
G^\dagger I_{\rm eff}^{-1}
\right).
\]
For correctly specified Bernoulli feedback, all three inequalities become equalities.
\end{proposition}

\begin{proof}
The assumed asymptotic linear expansion gives target covariance
$P_\theta M^{-1}\Omega M^{-1}P_\theta^\top$.
Since $LM=[I_{\rm eff}\;\;0]$, block elimination gives
$P_\theta M^{-1}=I_{\rm eff}^{-1}L$, which yields the stated sandwich covariance and, after the local Bregman expansion, the policy-risk limit.

Under mean correctness and $0\le Y_{nj}\le1$,
\[
\operatorname{Var}(Y_{nj}\mid v_{nj})
=
q_{nj}(1-q_{nj})
-
\mathbb E[Y_{nj}(1-Y_{nj})\mid v_{nj}]
\le
q_{nj}(1-q_{nj}).
\]
Conditional independence and summation of the score-covariance contributions give $\Omega\preceq M$. Congruence by $L$ gives $\Omega_{\rm eff}\preceq I_{\rm eff}$ and hence the risk bound. For Bernoulli responses, $Y_{nj}(1-Y_{nj})=0$ almost surely, so equality is recovered.
\end{proof}

\begin{theorem}[Least-favorable bounded soft feedback]
\label{thm:app-soft-least-favorable}
Fix an admissible design $\xi\in\Xi$ and suppose the conditions of Proposition~\ref{prop:app-soft-sandwich} hold under mean correctness with limiting root $\gamma^\dagger=\gamma_0$. Let $\Psi_{\rm soft}(\xi;P)$ denote the limiting sandwich policy-risk coefficient under a conditionally independent response law $P$ with the same conditional means and responses in $[0,1]$. The supremum over all such response laws for which the asymptotic expansion in Proposition~\ref{prop:app-soft-sandwich} holds is
\[
\sup_P \Psi_{\rm soft}(\xi;P)=\Phi_\kappa(\xi).
\]
The supremum is attained by conditionally Bernoulli feedback with the same conditional means. Consequently,
\[
\arg\min_{\xi\in\Xi}\sup_P\Psi_{\rm soft}(\xi;P)
=
\arg\min_{\xi\in\Xi}\Phi_\kappa(\xi).
\]
Thus the NAOD criterion minimizes the worst-case leading sandwich policy-risk coefficient over the bounded mean-correct soft-feedback class.
\end{theorem}

\begin{proof}
Fix $\xi$. Under mean correctness, the logistic curvature depends on the conditional means but not on the conditional variances of the responses. Hence all response laws considered in the theorem have the same limiting curvature $M_\kappa(\xi)$ and the same effective information $I_{\rm eff,\kappa}(\xi)$.

For any such response law $P$, Proposition~\ref{prop:app-soft-sandwich} gives
\[
\Omega_P\preceq M_\kappa(\xi).
\]
With $L=[I_d\;\;-C_\xi D_\xi^{-1}]$, congruence preserves the Loewner order and therefore
\[
\Omega_{{\rm eff},P}
=
L\Omega_P L^\top
\preceq
LM_\kappa(\xi)L^\top
=
I_{\rm eff,\kappa}(\xi).
\]
Since $G_0\succeq0$ and $I_{\rm eff,\kappa}(\xi)\succ0$,
\[
\Psi_{\rm soft}(\xi;P)
\le
\frac12
\operatorname{tr}\!\left(
G_0 I_{\rm eff,\kappa}(\xi)^{-1}
\right)
=
\Phi_\kappa(\xi).
\]

Now take conditionally Bernoulli responses with the same conditional means. Then
\[
\operatorname{Var}(Y_{nj}\mid v_{nj})
=
q_{nj}(\gamma_0)\{1-q_{nj}(\gamma_0)\},
\]
so the score covariance equals the logistic curvature. Hence
$\Omega_P=M_\kappa(\xi)$ and
$\Omega_{{\rm eff},P}=I_{\rm eff,\kappa}(\xi)$, which gives
$\Psi_{\rm soft}(\xi;P)=\Phi_\kappa(\xi)$.
The upper bound is therefore attainable and the first claim follows.

Because the equality
$\sup_P\Psi_{\rm soft}(\xi;P)=\Phi_\kappa(\xi)$
holds pointwise for every admissible design, minimizing both sides over $\Xi$ gives the stated equality of minimizer sets.
\end{proof}

The Arena experiment in Section~\ref{sec-arena-experiments} uses the archived soft target \(Y_J=p_A+\frac12p_T\in[0,1]\). NAOD therefore uses the same logistic curvature geometry as in the Bernoulli theory without identifying curvature with score covariance. Under mean correctness, Proposition~\ref{prop:app-soft-sandwich} gives the corresponding sandwich risk, while Theorem~\ref{thm:app-soft-least-favorable} shows that the population curvature criterion underlying NAOD is the worst-case leading sandwich policy-risk coefficient over bounded mean-correct soft feedback. Under mean misspecification, the sandwich result describes fluctuations around the joint population root rather than eliminating displacement of that root from the human target.

\section{Theory for Nuisance-Adjusted Optimal Design}
\label{app-naod-theory}

This appendix develops the theory used in Section~\ref{sec-problem} and Sections~\ref{sec-joint-model}--\ref{sec-policy-design}. We first connect acquisition-dependent residual exposure to target displacement and policy loss, then prove the conditional local minimax characterization and attainment of the NAOD criterion. We finally give the explicit separation from Target-info and formalize the target--nuisance coupling diagnostic used in the empirical analysis.

\subsection{Acquisition-Dependent Residual Exposure}
\label{app-exposure}

Section~\ref{sec-problem} defines the probability-scale judge discrepancy $\delta(z)=q(z)-p^\star(z)$ and the acquired residual score $b_\xi=\mathbb E_\xi[\delta(Z)X(Z)]$. The change-of-measure identity in \eqref{eq-residual-exposure} does not require a parametric model for the judge discrepancy. We now give its consequences for a nuisance-ignorant reward fit.

Suppose an interior population root $\theta_\xi$ satisfies $\mathbb E_\xi[(q-p_{\theta_\xi})X]=0$, and write $\Delta=\theta_\xi-\theta^\star$. Define the integrated statistical and policy curvatures
\[
H_\xi=\int_0^1\mathbb E_\xi\!\left[\sigma'\!\left(X^\top(\theta^\star+t\Delta)\right)XX^\top\right]dt,\qquad G_\xi=2\int_0^1(1-t)\nabla^2F(\theta_\xi-t\Delta)\,dt.
\]
Subtracting the population score equations and integrating the logistic derivative along the segment from $\theta^\star$ to $\theta_\xi$ gives $H_\xi\Delta=b_\xi$. Hence, whenever $H_\xi\succ0$,
\[
\theta_\xi-\theta^\star=H_\xi^{-1}b_\xi,\qquad D_F(\theta^\star,\theta_\xi)=\frac12\,b_\xi^\top H_\xi^{-1}G_\xi H_\xi^{-1}b_\xi.
\]
Thus policy impact depends on how the residual aligns with reward-score directions under the acquired law rather than on residual magnitude alone.

The joint target--nuisance model gives a complementary local interpretation. Holding the human parameter at $\theta_0$, define $b_\xi(a)=\sum_{i=1}^L\xi_i\{q_{\theta_0,a}(z_i)-p_{\theta_0}(z_i)\}X_i$. Differentiation gives $\partial b_\xi(a)/\partial a^\top=C_\xi(a)$, as in \eqref{eq-exposure-coupling}. Since $b_\xi(0)=0$, the fundamental theorem of calculus yields
\[
b_\xi(a)=\int_0^1 C_\xi(ta)a\,dt,\qquad b_\xi\!\left(a_0+\frac{\upsilon}{\sqrt n}\right)-b_\xi(a_0)=\frac{C_\xi(a_0)\upsilon}{\sqrt n}+O(n^{-1}).
\]
Hence the same cross-information that produces the Schur-complement adjustment in \eqref{eq-effective-information} also controls the local sensitivity of the exposed target score.

\subsection{Conditional Local Minimax Risk and Attainment}
\label{app-minimax}

We condition throughout on the historical information $\mathcal F_P$, so the representation $W$, allocation $\xi$, and all design quantities are fixed before the current outcomes are observed. Write $M=M_\kappa(\xi)$ as in \eqref{eq-joint-information} and $G_e=\operatorname{diag}(G_0,0)$.

\paragraph{Preliminary estimation and projected one-step estimator.}
For attainment, assume $\kappa>0$, a common positive lower bound on normalized trusted likelihood curvature over $\mathcal K_\theta$, a common positive lower bound on the weighted nuisance Gram matrix $\sum_i(n_i/n)W_iW_i^\top$, and a fixed positive distance of the local truths from the boundary of $\mathcal K=\mathcal K_\theta\times\mathcal K_a$. Let $c_s>0$ be smaller than one half of a valid asymptotic lower bound on $\lambda_{\min}(M)$.

Let $\widetilde\theta_n$ minimize the trusted negative log likelihood on $\mathcal K_\theta$. Holding its target contribution fixed as an offset, let
\[
\widetilde a_n\in\arg\min_{a\in\mathcal K_a}\sum_{i=1}^L\sum_{\ell=1}^{n_i}\left\{\operatorname{sp}(X_i^\top\widetilde\theta_n+W_i^\top a)-Y_{i\ell}(X_i^\top\widetilde\theta_n+W_i^\top a)\right\}.
\]
Set $\widetilde\gamma_n=(\widetilde\theta_n^\top,\widetilde a_n^\top)^\top$. If the declared nuisance-curvature check fails, a fixed interior anchor may be used. Under the stated conditions, this event is asymptotically negligible. The guarded one-step update is
\[
\gamma_n^\circ=
\begin{cases}
\widetilde\gamma_n+I_n(\widetilde\gamma_n)^{-1}U_n(\widetilde\gamma_n), & \lambda_{\min}\{I_n(\widetilde\gamma_n)/n\}\ge c_s,\\
\widetilde\gamma_n, & \text{otherwise},
\end{cases}
\qquad
\widehat\gamma_n=\operatorname{Proj}_{\mathcal K}(\gamma_n^\circ),
\]
with target output $\widehat\theta_n=P_\theta\widehat\gamma_n$. The preliminary and one-step update may use the same current observations. No independence between them is assumed.

Strong convexity of the trusted likelihood and the bounded-score moment bound from Appendix~\ref{app-foundations-lan} give $\|\widetilde\theta_n-\theta_{n,h}\|_{L^4}=O(n^{-1/2})$. At the true nuisance coefficient, the normalized offset score is a centered judge-score average plus a term Lipschitz in $\widetilde\theta_n-\theta_{n,h}$. The nuisance Gram lower bound and the common interior probability bound therefore give $\|\widetilde a_n-a_{n,\upsilon}\|_{L^4}=O(n^{-1/2})$, and hence $\|\widetilde\gamma_n-\gamma_{n,\eta}\|_{L^4}=O(n^{-1/2})$ uniformly on every fixed local parameter ball.

To obtain the one-step expansion, put $e_n=\widetilde\gamma_n-\gamma_{n,\eta}$, $\widehat H_n=I_n(\widetilde\gamma_n)/n$, and $\overline H_n=\int_0^1 I_n(\gamma_{n,\eta}+te_n)/n\,dt$. Since the derivative of the score is minus the information, on the guard event the Newton update satisfies the exact identity
\[
\sqrt n(\gamma_n^\circ-\gamma_{n,\eta})=\widehat H_n^{-1}\frac{U_n(\gamma_{n,\eta})}{\sqrt n}+\widehat H_n^{-1}(\widehat H_n-\overline H_n)\sqrt n\,e_n.
\]
Normalized Hessians are uniformly Lipschitz on the compact neighborhood, so $\|\widehat H_n-\overline H_n\|_{\mathrm{op}}\le C\|e_n\|$. The second term is therefore $O_{L^2}(n^{-1/2})$. Hessian concentration from Appendix~\ref{app-foundations-lan}, local information convergence, and the fixed spectral guard permit replacing $\widehat H_n^{-1}$ by $M^{-1}$. The guard complement has vanishing squared contribution by the fourth-moment bound, and the fixed boundary margin makes the final projection asymptotically inactive in mean square. Consequently,
\[
\sqrt n(\widehat\gamma_n-\gamma_{n,\eta})=M^{-1}\frac{U_n(\gamma_{n,\eta})}{\sqrt n}+o_{L^2}(1)
\]
uniformly on bounded local parameter sets.

The same argument will be used in Appendix~\ref{app-representation-finite} around a nearby population root rather than the correctly specified center. More generally, if an interior population root $\gamma_n^\dagger$ has normalized expected Hessian and score covariance converging to the same $M\succ0$, and a preliminary satisfies $r_n=\|\widetilde\gamma_n-\gamma_n^\dagger\|_{L^4}=o(n^{-1/4})$, then the nonlinear Newton remainder is $O_{L^2}(\sqrt n\,r_n^2)=o(1)$ and
\[
\sqrt n(\widehat\gamma_n-\gamma_n^\dagger)=M^{-1}\frac{U_n(\gamma_n^\dagger)}{\sqrt n}+o_{L^2}(1).
\]

\paragraph{Conditional local minimax lower bound.}
We now prove the lower-bound part of Theorem~\ref{thm-local-policy-risk}. Assume first $G_0\succ0$. Fix $\varepsilon\in(0,1)$ and a bounded local parameter set. Continuity of the policy Hessian gives a neighborhood of $\theta_0$ on which $\nabla^2F\succeq(1-\varepsilon)G_0$. Project the rescaled output of an arbitrary reward estimator onto a fixed $G_0$-metric ball containing the target local parameters. Metric projection cannot increase its quadratic distance to any point in this ball. For decisions outside the local neighborhood, convexity together with positive local curvature gives a fixed positive policy loss, whereas the projected quadratic loss remains bounded. Hence, for all sufficiently large $n$,
\[
nD_F\!\left(\theta_0+\frac{h}{\sqrt n},T_n\right)\ge\frac{1-\varepsilon}{2}\|\overline T_n-h\|_{G_0}^2,
\]
where $\overline T_n$ denotes the corresponding localized rescaled decision.

Place a smooth product prior on $\eta$ supported on $(-t,t)^k$, with prior Fisher information $\pi^2t^{-2}I_k$. Let $E_n=\overline T_n-P_\theta\eta$ and let $S_n$ be the joint score of the conditional likelihood and this local prior. Integration by parts in the local parameter gives $\mathbb E[E_nS_n^\top]=P_\theta$. Writing $J_n=\mathbb E[S_nS_n^\top]$, positive semidefiniteness of
\[
\mathbb E\!\left[(E_n-P_\theta J_n^{-1}S_n)(E_n-P_\theta J_n^{-1}S_n)^\top\right]
\]
therefore gives
\[
\mathbb E[E_nE_n^\top]\succeq P_\theta J_n^{-1}P_\theta^\top,\qquad J_n=\mathbb E_\eta\mathcal I_n(\eta)+\pi^2t^{-2}I_k,
\]
where $\mathcal I_n(\eta)$ is Fisher information for the local coordinate $\eta$. The regularity conditions in Appendix~\ref{app-foundations-notation} imply that $\mathcal I_n(\eta)\to M$ uniformly on every fixed prior support. Taking the trace against $G_0$, using worst-case risk to dominate prior-averaged risk, and then letting $n\to\infty$, $t\to\infty$, and $\varepsilon\downarrow0$ yields \eqref{eq-local-risk-lower-bound}, because \eqref{eq-app-block-inverse} gives $P_\theta M^{-1}P_\theta^\top=I_{\mathrm{eff},\kappa}^{-1}$. When $G_0\succeq0$ is singular, Proposition~\ref{prop:app-policy-null} removes the exact policy-null coordinates and the same argument applies on the policy-relevant quotient.

\paragraph{Attainment.}
Under correct specification, $U_n(\gamma_{n,\eta})/\sqrt n$ is centered, its covariance converges to $M$, and its fourth moments are uniformly bounded. Combining the one-step expansion with the local Bregman expansion in \eqref{eq-app-local-policy-loss} and the resulting uniform integrability gives
\[
n\,\mathbb E_\eta D_F(\theta_{n,h},\widehat\theta_n)\longrightarrow\frac12\operatorname{tr}(G_eM^{-1})=\Phi_\kappa(\xi)
\]
uniformly on every fixed local parameter ball. Thus the projected one-step estimator attains the lower-bound constant in Theorem~\ref{thm-local-policy-risk}.

\subsection{Separation and Target--Nuisance Coupling}
\label{app-coupling}

\begin{proposition}[Separation from Target-info]
\label{prop-design-separation}
Let $\lambda>1$, and consider the scalar two-type model with $(X_1,X_2)=(2\lambda,2)$, $(W_1,W_2)=(2,-2\lambda)$, $\theta_0=a_0=0$, and $H_c=(\lambda^2+1)/2$. For unrestricted allocation with $\kappa>0$ and $G_0>0$, Target-info selects $\xi_{\mathrm T}=(1,0)$, whereas NAOD selects $\xi^\star=(\lambda,1)/(\lambda+1)$. Their leading risk ratio is
\[
\frac{\Phi_\kappa(\xi_{\mathrm T})}{\Phi_\kappa(\xi^\star)}=1+\frac{2(\lambda^2+1)}{\kappa(\lambda+1)^2}>1.
\]
\end{proposition}

\begin{proof}
At the zero center, every judge variance is $1/4$. If $x\in[0,1]$ denotes the fraction of judge labels assigned to the first type, the two judge information atoms are
\[
J_1=\begin{pmatrix}\lambda^2&\lambda\\\lambda&1\end{pmatrix},\qquad J_2=\begin{pmatrix}1&-\lambda\\-\lambda&\lambda^2\end{pmatrix}.
\]
The raw target information is $\kappa H_c+x\lambda^2+(1-x)$, which is strictly increasing in $x$ because $\lambda>1$. Hence Target-info selects $x=1$.

For the judge block $J(x)=xJ_1+(1-x)J_2$, the nuisance entry is $D(x)=x+\lambda^2(1-x)$ and $\det J(x)=x(1-x)(\lambda^2+1)^2$. Because trusted information enters only the target block, the effective target information is
\[
I_{\mathrm{eff}}(x)=\kappa H_c+\frac{x(1-x)(\lambda^2+1)^2}{x+\lambda^2(1-x)}.
\]
The derivative of the second term has roots $\lambda/(\lambda+1)$ and $\lambda/(\lambda-1)$. Only the first lies in $[0,1]$, and the derivative changes from positive to negative there. Thus NAOD selects $x^\star=\lambda/(\lambda+1)$. At this allocation, $I_{\mathrm{eff}}(x^\star)=\kappa H_c+(\lambda^2+1)^2/(\lambda+1)^2$, whereas $I_{\mathrm{eff}}(1)=\kappa H_c$. Since the scalar leading policy risk is $G_0/\{2I_{\mathrm{eff}}(x)\}$ by \eqref{eq-policy-design-criterion}, substituting $H_c=(\lambda^2+1)/2$ gives the stated ratio.
\end{proof}

This construction isolates the acquisition criterion because both rules use the same correctly specified joint model and final estimator, while Target-info favors the type with the largest raw target block even when that type cannot separate target from nuisance variation.

\paragraph{Target--nuisance coupling.}
For any positive-definite design, define
\[
K_\xi=A_\kappa(\xi)^{-1/2}C_\xi D_\xi^{-1/2},\qquad \rho^2(\xi)=\|K_\xi\|_{\mathrm{op}}^2.
\]
Positive definiteness of the Schur complement implies $0\le\rho^2(\xi)<1$. By the whitening identity in \eqref{eq-app-whitened-efficient}, if $\rho_j^2$ are the eigenvalues of $K_\xi K_\xi^\top$, then $1-\rho_j^2$ are the corresponding fractions of whitened target information remaining after nuisance adjustment. Thus $\rho^2(\xi)$ records the strongest target--nuisance coupling. Values near one indicate a target-score direction that is nearly reproducible by nuisance variation and for which raw Target-info can substantially overstate usable target information.

This coupling measure is geometric rather than policy-weighted, so its consequence for policy risk also depends on whether the coupled target directions receive substantial weight under $G_0$. The Arena analysis in Figure~\ref{fig-arena-coupling} and Table~\ref{tab:arena-judge} evaluates $\rho^2$ for the Target-info-selected design and uses it as a diagnostic of target--nuisance confounding rather than as a finite-sample prediction of the realized NAOD gain.

\section{Finite-Pool Selection and Computation}
\label{app-finite-pool}

This appendix gives the optimization, statistical, and computational details for the finite-pool procedure in Section~\ref{sec-finite-pool}. We first justify the convex relaxation and its executable optimality certificate, then prove Theorem~\ref{thm-distinct-comparisons} directly for the realized set of distinct comparisons. We finally separate plug-in error from numerical optimization error and record the computational cost of the implementation.

\subsection{Finite-Pool Optimization and Certificates}
\label{app-finite-pool-optimization}

For the fitted finite-pool criterion in \eqref{eq-finite-pool-criterion}, write $\widehat{\mathcal I}_x=\widehat{\mathcal I}(x)$. The seed $S_0$ is chosen only when needed to keep fitted information positive definite. Because all information atoms are positive semidefinite, positive definiteness of the trusted-plus-seed information then implies $\widehat{\mathcal I}_x\succ0$ for every $x\in\mathcal X_B$.

The map $M\mapsto\tfrac12\operatorname{tr}(GM^{-1})$ is convex for $G\succeq0$ on the positive-definite cone. Indeed, for a symmetric perturbation $H$,
\[
\frac{d^2}{dt^2}\left.\frac12\operatorname{tr}\!\left\{G(M+tH)^{-1}\right\}\right|_{t=0}
=
\operatorname{tr}\!\left(GM^{-1}HM^{-1}HM^{-1}\right)
\ge0.
\]
Since $\widehat{\mathcal I}(x)$ is affine in $x$, $f_B$ is convex on $\mathcal X_B$. Differentiation gives the gradient used by the linear oracle.

For a feasible relaxed point $x$, define the full Frank--Wolfe gap
\[
g_{\mathrm{FW}}(x)
=
\max_{u\in\mathcal X_B}
\nabla f_B(x)^\top(x-u).
\]
Convexity gives $0\le f_B(x)-\min_{u\in\mathcal X_B}f_B(u)\le g_{\mathrm{FW}}(x)$. The word full is important because the linear oracle must optimize over the complete declared feasible set. Under the group constraints in Section~\ref{sec-finite-pool}, this oracle is exact. After fixing the seed coordinates, it selects the smallest gradient coefficient in each available group and then the required number of groups with the smallest such coefficients.

Let $\widehat x$ be the final relaxed iterate and let $S$ be the returned feasible subset after rounding and exchanges. Write $f_B(S)=f_B(\mathbf 1_S)$, and let $f^\star_{B,\mathrm{int}}$ denote the minimum over feasible integer selections. Since every integer design is feasible for the relaxation,
\begin{equation}
0
\le
f_B(S)-f^\star_{B,\mathrm{int}}
\le
f_B(S)-f_B(\widehat x)+g_{\mathrm{FW}}(\widehat x).
\label{eq-selection-certificate}
\end{equation}
Thus the final Frank--Wolfe gap and the actual rounding-and-exchange change give a computable certificate for the returned subset under the frozen fitted criterion. This is an optimization certificate, not a statistical confidence interval and not a certificate that the fitted information equals population information. The complete optimization and integerization procedure is summarized in Algorithm~\ref{alg-naod}.

\begin{algorithm}[t]
\caption{NAOD finite-pool acquisition}
\label{alg-naod}
\begin{algorithmic}[1]
\Statex \textbf{Input} Fitted information atoms $\{\widehat J_i\}_{i=1}^N$, policy weight $\widehat G_e$, trusted information $\widehat I_c$, budget $B$, feasible set $\mathcal X_B$, seed $S_0$, tolerance $\varepsilon>0$, and iteration limit $T_{\max}$.
\State Choose $x\in\mathcal X_B$ with $\widehat{\mathcal I}(x)\succ0$.
\For{$t=1,\ldots,T_{\max}$}
    \State Compute $s_t\in\operatorname*{arg\,min}_{u\in\mathcal X_B}\nabla f_B(x)^\top u$.
    \If{$\nabla f_B(x)^\top(x-s_t)\le\varepsilon$}
        \State \textbf{break}
    \EndIf
    \State Choose $\alpha_t\in\operatorname*{arg\,min}_{\alpha\in[0,1]}f_B(x+\alpha(s_t-x))$.
    \State Update $x\gets x+\alpha_t(s_t-x)$.
\EndFor
\State Set $\widehat x\gets x$ and recompute $g_{\mathrm{FW}}(\widehat x)$.
\State Round $\widehat x$ by largest weights to a feasible set $S$ of size $B$, including $S_0$ and respecting all group capacities.
\State Apply feasible exchanges that decrease $f_B(S)$ and preserve positive-definite fitted information.
\State Freeze $S$ before observing its judge labels.
\State \Return $S$ and the certificate in \eqref{eq-selection-certificate}.
\end{algorithmic}
\end{algorithm}

The final gap is recomputed regardless of whether optimization stops by tolerance or by the iteration limit. Reaching the limit is not treated as convergence. Improving exchanges can only reduce the certificate because \eqref{eq-selection-certificate} is evaluated at the final returned subset. The selected IDs and their fitted information are frozen before any selected judge outcome is revealed.

\subsection{Statistical Guarantees for Distinct Comparisons}
\label{app-distinct-comparisons}

We now prove Theorem~\ref{thm-distinct-comparisons}. Unlike the fixed-support analysis in Appendix~\ref{app-minimax}, the selected comparisons need not repeat a finite set of feature types. The proof therefore retains each selected row and works with the realized information matrix $M_B$ appearing in the theorem.

\paragraph{Distinct-comparison regularity.}
Condition on the historical data, candidate features, and the selected set $S_B$, all fixed before the selected outcomes are observed. The dimensions $d$ and $r$ are fixed. Selected and trusted feature rows are uniformly bounded. Current labels are conditionally independent and correctly specified under the joint Bernoulli model. The parameter centers lie a fixed positive distance from the boundary of a common compact product set. Assume $n_{c,B}/B$ is bounded above and away from zero, trusted likelihood curvature has a common positive lower bound, and $\lambda_{\min}(M_B)\ge\mu>0$. The guards, projection region, and derivative bounds are the same as in Appendix~\ref{app-minimax}.

For $\gamma_{B,\eta}=\gamma_0+\eta/\sqrt B$, let $M_B(\eta)$ be the expected joint information divided by $B$ at the local truth. Bounded rows and bounded logistic derivatives imply, for every fixed $R<\infty$,
\[
\sup_{\|\eta\|\le R}\|M_B(\eta)-M_B\|_{\mathrm{op}}=O_R(B^{-1/2}).
\]
No convergence of $M_B$ is needed for this comparison.

Let $U_B$ denote the total joint score from the trusted block and the selected judge records, and define $\Delta_{B,\eta}=U_B(\gamma_{B,\eta})/\sqrt B$. There are $O(B)$ independent bounded score summands, so $\Delta_{B,\eta}$ has uniformly bounded fourth moments on bounded local parameter sets, and the normalized observed information has $L^4$ fluctuation $O(B^{-1/2})$ when trusted covariates are random. Trusted curvature gives the target preliminary an $O_{L^4}(B^{-1/2})$ error. Moreover, the nuisance principal block of $M_B$ is at least $\mu I_r$. Since logistic variances are at most $1/4$, $B^{-1}\sum_{i\in S_B}W_{Bi}W_{Bi}^\top\succeq4\mu I_r$, which supplies the strong convexity required by the offset nuisance fit. The full preliminary is therefore root-$B$ in $L^4$.

Applying the exact Newton expansion from Appendix~\ref{app-minimax} with the finite-array information retained gives
\[
\sqrt B(\widehat\gamma_B-\gamma_{B,\eta})=M_B(\eta)^{-1}\Delta_{B,\eta}+o_{L^2}(1)
\]
uniformly over bounded $\eta$ and eligible frozen selected arrays. The guard complement and final projection are negligible by the common information floor, fourth-moment bounds, and boundary margin.

Under correct specification, $\operatorname{Cov}(\Delta_{B,\eta})=M_B(\eta)$. Hence the leading expected policy quadratic form is $\tfrac12\operatorname{tr}\{G_eM_B(\eta)^{-1}\}$. The preceding $O_R(B^{-1/2})$ information comparison, the inverse-information bound, and the local Bregman expansion in \eqref{eq-app-local-policy-loss} then yield exactly \eqref{eq-distinct-policy-risk}, uniformly over bounded local parameters. This proves the risk statement in Theorem~\ref{thm-distinct-comparisons} without requiring the selected empirical proportions to approach a fixed-support allocation.

If $M_B\to M_\infty\succ0$, the bounded triangular-array central limit theorem and the same third-order likelihood expansion as in Proposition~\ref{prop:app-lan} give a LAN experiment with information $M_\infty$. The localized lower-bound argument of Appendix~\ref{app-minimax} therefore gives the corresponding conditional local minimax bound. If $G_0$ is singular, Proposition~\ref{prop:app-policy-null} applies on the policy-relevant quotient.

The theorem is conditional on the realized selected set and uses its true information. It therefore preserves the statistical effect of rounding, but it does not assert that every rounded set satisfies the information floor, nor does fitted positive definiteness imply true positive definiteness. Those are separate plug-in questions addressed next.

\subsection{Plug-In Error, Numerical Accuracy, and Computation}
\label{app-computation}

\paragraph{Plug-in error.}
Define the population counterpart by $f_B^0(x)=\tfrac12\operatorname{tr}\{G_e\mathcal I_0(x)^{-1}\}$, where $\mathcal I_0(x)=I_c+\sum_{i=1}^N x_iJ_i$, and recall that $f_B$ in \eqref{eq-finite-pool-criterion} uses the fitted quantities $\widehat G_e$ and $\widehat{\mathcal I}(x)$. Suppose both information matrices have eigenvalues at least $\underline\mu_B>0$ over the feasible family. The inverse identity and the trace inequality give
\[
|f_B(x)-f_B^0(x)|
\le
\frac{\|\widehat G_e-G_e\|_*}{2\underline\mu_B}
+
\frac{\operatorname{tr}(G_e)}{2\underline\mu_B^2}
\|\widehat{\mathcal I}(x)-\mathcal I_0(x)\|_{\mathrm{op}}.
\]
For example, bounds $\|\widehat I_c-I_c\|_{\mathrm{op}}\le\delta_c$ and $\max_i\|\widehat J_i-J_i\|_{\mathrm{op}}\le\delta_J$ imply $\sup_{x\in\mathcal X_B}\|\widehat{\mathcal I}(x)-\mathcal I_0(x)\|_{\mathrm{op}}\le\delta_c+B\delta_J$ because every feasible design has total mass $B$.

Let $\Delta_B=\sup_{x\in\mathcal X_B}|f_B(x)-f_B^0(x)|$. If $S_B^{0,\star}$ minimizes the population criterion over feasible integer selections, then the fitted optimization certificate and two applications of the uniform plug-in bound give
\[
0
\le
f_B^0(S)-f_B^0(S_B^{0,\star})
\le
2\Delta_B
+
f_B(S)-f_B(\widehat x)
+
g_{\mathrm{FW}}(\widehat x).
\]
Thus plug-in error and numerical/rounding error enter separately. A small Frank--Wolfe certificate establishes accurate optimization of the fitted objective. It does not by itself establish that the fitted objective accurately represents the population criterion or that the joint model is correctly specified.

\paragraph{Numerical accuracy of the estimator.}
The statistical proofs in Appendix~\ref{app-minimax} are written for exact trusted and offset preliminary minimizers, but approximate convex solutions preserve the required rates under a quantitative stopping rule. For a normalized convex objective $Q$ with Hessian bounded below by $cI$, let $t^\star$ be its constrained minimizer and define the global first-order gap $g_Q(t)=\max_{u\in\mathcal K}\nabla Q(t)^\top(t-u)$. Strong convexity and convexity give
\[
\frac c2\|t-t^\star\|^2\le Q(t)-Q(t^\star)\le g_Q(t).
\]
Hence gaps of order \(O(n_c^{-1})\) for the normalized trusted objective and \(O(B^{-1})\) for the normalized nuisance objective yield numerical errors of order \(O(n_c^{-1/2})\) and \(O(B^{-1/2})\), respectively, or smaller, leaving the one-step expansions unchanged. Selection line searches need not themselves prove convergence because the final Frank--Wolfe gap is recomputed at the returned relaxed point and is the quantity entering \eqref{eq-selection-certificate}.

\paragraph{Computation.}
For $k=d+r$, each fitted information atom has rank-one form $\widehat J_i=\widehat t_i\widehat v_i\widehat v_i^\top$. Storing $(\widehat t_i,\widehat v_i)$ rather than a dense $k\times k$ matrix for every candidate requires $O(Nk+k^2)$ memory. After factorizing $\widehat{\mathcal I}_x$, the gradient coordinate can be evaluated as
\[
\frac{\partial f_B(x)}{\partial x_i}
=
-\frac12\widehat t_i\,\widehat v_i^\top\widehat{\mathcal I}_x^{-1}\widehat G_e\widehat{\mathcal I}_x^{-1}\widehat v_i.
\]
A dense full information-and-gradient evaluation costs $O(Nk^2+k^3)$, with the $k^3$ term coming from the factorization. The implementation uses positive-definite factorizations and linear solves rather than forming matrix inverses explicitly. Total selection time additionally depends on the number of Frank--Wolfe iterations, line-search evaluations, rounding, and improving exchanges.

\section{Learning the Judge-Deviation Representation}
\label{app-representation-learning}

This appendix proves the representation-learning results in Sections~\ref{sec-representation-cost} and~\ref{sec-design-reversals}. We first characterize the target effect of representation error and the exact cancellation of modeled nuisance directions, then prove the finite-sample representation cost in Theorem~\ref{thm-representation-cost} and its design-ranking consequence. We finally give the higher-order population-root expansion that clarifies the limit of first-order cancellation.

\subsection{Representation Error and Oracle Recovery}
\label{app-representation-error}

Recall the learned-representation residual $e_m$ in \eqref{eq-representation-residual} and the target influence operator $B_m$ and policy kernel $R_m$ in \eqref{eq-representation-influence}. We first justify the first-order interpretation of these quantities.

\begin{proposition}[Local representation-error expansion]
\label{prop:app-representation-root}
Fix the learned representation, nuisance center, and allocation, and suppose the regularity conditions of Appendices~\ref{app-foundations-notation} and~\ref{app-minimax} hold with positive joint information $M_m$. If the true judge logits equal the nominal logits at $\gamma_{0m}$ plus a sufficiently small deterministic vector $e\in\mathbb R^L$, while the trusted model remains correctly specified at $\theta_0$, then the fitted population score has a unique local root $\gamma_m^\dagger(e)$ satisfying
\[
\gamma_m^\dagger(e)-\gamma_{0m}=M_m^{-1}\sum_{i=1}^L\xi_{m,i}t_{m,i}v_{m,i}e_i+O(\|e\|^2).
\]
Consequently,
\[
P_\theta\{\gamma_m^\dagger(e)-\gamma_{0m}\}=B_me+O(\|e\|^2).
\]
\end{proposition}

\begin{proof}
Let $\Psi_m(\gamma,e)$ denote the normalized expected fitted score when the true judge logit at type $i$ is $v_{m,i}^\top\gamma_{0m}+e_i$, while the trusted probabilities remain those of $\theta_0$. At $(\gamma_{0m},0)$, the score is zero, its derivative with respect to $\gamma$ is $-M_m$, and its derivative with respect to $e$ applied to a direction $e$ is $\sum_i\xi_{m,i}t_{m,i}v_{m,i}e_i$. The inverse-information bound and bounded logistic derivatives give a local implicit-function expansion with the stated linear term and an $O(\|e\|^2)$ remainder. Taking target coordinates and using the definition of $B_m$ in \eqref{eq-representation-influence} gives the second display.
\end{proof}

The operator $B_m$ annihilates representation errors that lie in the learned nuisance span. In particular, for any $c\in\mathbb R^r$,
\[
B_m\widehat W_m c
=
P_\theta M_m^{-1}\sum_{i=1}^L\xi_{m,i}t_{m,i}v_{m,i}\widehat W_{m,i}^\top c
=
P_\theta M_m^{-1}M_m\binom{0}{c}
=
0,
\]
so $B_m\widehat W_m=0$. Thus an error component in the column span of $\widehat W_m$ has no first-order effect on the target parameter. More strongly, if $e=\widehat W_m c$ and $a_{0m}+c$ remains in the admissible interior neighborhood, then $\gamma_{0m}+(0^\top,c^\top)^\top$ reproduces every perturbed judge logit exactly while leaving the trusted model unchanged. Such a represented error is therefore absorbed entirely by the nuisance coefficient, although estimating that coefficient still reduces effective target information through \eqref{eq-effective-information}.

Proposition~\ref{prop:app-representation-root} also gives a generic oracle-recovery condition. Along a deterministic sequence of learned representations and designs, if $\|e_m\|=o(n^{-1/4})$ and $\sqrt n\,\|B_me_m\|\to0$, then the target population-root displacement is $o(n^{-1/2})$. The nearby-root one-step expansion in Appendix~\ref{app-minimax} then recovers the oracle leading policy risk. For random representation errors, expected-risk recovery additionally requires the moment control used below.

\subsection{Finite-Sample Representation Cost and Design Reversals}
\label{app-representation-finite}

\paragraph{Representation-learning regularity.}
Let $\mathcal F_m$ denote the historical information used to learn the judge-deviation representation. Conditional on $\mathcal F_m$, the aligned representation $\widehat W_m$, nuisance center $a_{0m}$, and realized allocation proportions $\xi_{m,i}$ are fixed before the current outcomes are observed. We assume fixed support and parameter dimensions, common bounded feature and logit ranges, a common compact product parameter set whose nominal centers $\gamma_{0m}=(\theta_0^\top,a_{0m}^\top)^\top$ remain a fixed positive distance from the boundary, and the trusted-curvature, nuisance-curvature, guard, and projection conditions of Appendix~\ref{app-minimax} with common constants. Conditional on $\mathcal F_m$, current judge outcomes are independent Bernoulli variables with their actual probabilities and are independent of a correctly specified trusted block. The trusted-to-judge sample ratio is bounded above and away from zero, and $M_m\succeq\mu I$ for some fixed $\mu>0$. In addition to the convergence assumptions in Theorem~\ref{thm-representation-cost}, assume $\sup_m\mathbb E\|\sqrt m\,e_m\|^4<\infty$.

\begin{proof}[Proof of Theorem~\ref{thm-representation-cost}]
Condition on $\mathcal F_m$. On the event that $\|e_m\|$ lies in the common local neighborhood, Proposition~\ref{prop:app-representation-root} gives a population root $\gamma_m^\dagger$ satisfying
\[
\gamma_m^\dagger-\gamma_{0m}=M_m^{-1}\sum_{i=1}^L\xi_{m,i}t_{m,i}v_{m,i}e_{m,i}+O(\|e_m\|^2).
\]
The fourth-moment assumption gives $\mathbb P(\|e_m\|>\varepsilon)=O(m^{-2})$ for every fixed sufficiently small $\varepsilon>0$. Since $n/m$ is bounded and policy loss is bounded on the common compact parameter set, these exceptional historical samples contribute $o(1)$ to the expected policy loss after multiplication by $n$.

On the local event, let
\[
\Delta_m=\frac{U_n(\gamma_m^\dagger)}{\sqrt n},\qquad A_m^\dagger=\frac1n\mathbb E[I_n(\gamma_m^\dagger)\mid\mathcal F_m].
\]
Because $\gamma_m^\dagger$ is the conditional population root, $\mathbb E[\Delta_m\mid\mathcal F_m]=0$. Bounded logistic derivatives and Proposition~\ref{prop:app-representation-root} imply $A_m^\dagger=M_m+O(\|e_m\|)$ and $\operatorname{Cov}(\Delta_m\mid\mathcal F_m)=M_m+O(\|e_m\|)$.

The trusted-plus-offset preliminary has $L^4$ distance $O(n^{-1/2}+m^{-1/2})$ from $\gamma_m^\dagger$. Since $n/m$ is bounded, the Newton argument of Appendix~\ref{app-minimax}, retaining the conditional matrix $A_m^\dagger$, gives
\[
\sqrt n(\widehat\gamma_{n,m}-\gamma_m^\dagger)=(A_m^\dagger)^{-1}\Delta_m+o_{L^2}(1).
\]
Combining this expansion with the population-root expansion and taking target coordinates yields
\[
\sqrt n(\widehat\theta_{n,m}-\theta_0)=P_\theta(A_m^\dagger)^{-1}\Delta_m+\sqrt n\,B_me_m+o_{L^2}(1).
\]
Indeed, the omitted quadratic root remainder is negligible because $\sqrt n\,\|e_m\|_{L^4}^2=O(\sqrt n/m)=o(1)$. The first leading term has conditional mean zero given $\mathcal F_m$, whereas $B_me_m$ is $\mathcal F_m$-measurable, so their cross second moment vanishes.

Using the local policy expansion in \eqref{eq-app-local-policy-loss}, the common Hessian bound, and the fourth-moment controls therefore gives the finite-sequence identity
\[
n\,\mathbb E D_F(\theta_0,\widehat\theta_{n,m})=\mathbb E\Phi_m+\frac n2\,\mathbb E[e_m^\top R_me_m]+o(1).
\]
This identity separates current-sample variation from representation-induced displacement before taking the historical-sample limit.

Write $Z_m=\sqrt m\,e_m$. The information floor uniformly bounds $\Phi_m$ and $\|R_m\|_{\mathrm{op}}$, so $\Phi_m\to\Phi_0$ in probability implies $\mathbb E\Phi_m\to\Phi_0$. Moreover, $Z_m\Rightarrow Z_P$ and $R_m\to R_0$ in probability imply $Z_m^\top R_mZ_m\Rightarrow Z_P^\top R_0Z_P$. The uniform fourth-moment bound makes these quadratic forms uniformly integrable, and hence
\[
\mathbb E[Z_m^\top R_mZ_m]\longrightarrow\mathbb E[Z_P^\top R_0Z_P]=\operatorname{tr}(R_0\Sigma_P)+b_P^\top R_0b_P.
\]
Since $n/m\to\lambda_P$, substitution into the preceding finite-sequence identity and division by $n$ give exactly the risk expansion in \eqref{eq-representation-risk}.
\end{proof}

The same expansion gives the ranking result in Section~\ref{sec-design-reversals}. Applying \eqref{eq-representation-risk} to NAOD and Target-info under the same representation-learning protocol, label budgets, and final estimator, and subtracting the two risks, gives \eqref{eq-design-reversal}. Thus the oracle ordering is recovered when $n/m\to0$, whereas for $n/m\to\lambda_P>0$ the design-dependent representation cost can contribute at the same leading order as the oracle sampling risk.

\subsection{Higher-Order Representation Effects}
\label{app-representation-higher-order}

First-order influence does not by itself characterize finite representation error. We therefore record the second-order population-root expansion that determines when cancellation of $B_me_m$ is sufficient for oracle recovery.

For a fixed regular design, let $V=[v_1,\ldots,v_L]$, $\Lambda_\xi=\operatorname{diag}(\xi_1,\ldots,\xi_L)$, $T=\operatorname{diag}(t_1,\ldots,t_L)$, and define $\mathcal L_\xi=M^{-1}V\Lambda_\xi T$. For $v_c=(X_c^\top,0^\top)^\top$, define
\[
\begin{aligned}
Q_J(e,f)
&=
\sum_{i=1}^L\xi_i\sigma''(v_i^\top\gamma_0)v_i e_if_i,\\
Q_M(h,g)
&=
\sum_{i=1}^L\xi_i\sigma''(v_i^\top\gamma_0)v_i(v_i^\top h)(v_i^\top g)
+\kappa\,\mathbb E\!\left[\sigma''(X_c^\top\theta_0)v_c(v_c^\top h)(v_c^\top g)\right].
\end{aligned}
\]

\begin{proposition}[Second-order population-root expansion]
\label{prop:app-second-order-root}
Suppose the true judge logits are $v_i^\top\gamma_0+e_i$, while the trusted model remains correctly specified at $\theta_0$. For sufficiently small $e$, the nearby fitted population root satisfies
\[
\gamma^\dagger(e)-\gamma_0
=
s_1(e)+s_2(e)+O(\|e\|^3),
\]
\[
s_2(e)
=
\frac12M^{-1}\!\left\{
Q_J(e,e)-Q_M(s_1(e),s_1(e))
\right\},
\]
where $s_1(e)=\mathcal L_\xi e$.
If $e=Wb$ and $a_0+b$ remains in the admissible interior neighborhood, then $\gamma^\dagger(e)=\gamma_0+(0^\top,b^\top)^\top$ exactly, so the target displacement vanishes to every order and in particular $s_2(Wb)=0$.
\end{proposition}

\begin{proof}
Let $\ell_i=v_i^\top\gamma_0$. The normalized population score at $\gamma_0+h$ under logit perturbation $e$ is
\[
\Psi(h,e)=\sum_{i=1}^L\xi_i\{\sigma(\ell_i+e_i)-\sigma(\ell_i+v_i^\top h)\}v_i+\kappa\,\mathbb E\!\left[\{\sigma(X_c^\top\theta_0)-\sigma(X_c^\top\theta_0+v_c^\top h)\}v_c\right].
\]
At $(h,e)=(0,0)$, the derivatives with respect to $h$ and $e$ are $-M$ and $V\Lambda_\xi T$, respectively, so the local root satisfies $h=O(\|e\|)$. A second-order Taylor expansion of the true-mean and fitted-mean terms gives
\[
0=V\Lambda_\xi T e+\frac12Q_J(e,e)-Mh-\frac12Q_M(h,h)+O(\|e\|^3+\|h\|^3).
\]
The first-order relation gives $h-s_1(e)=O(\|e\|^2)$. Bilinearity of $Q_M$ then allows $Q_M(h,h)$ to be replaced by $Q_M(s_1,s_1)$ at third-order error, yielding the stated expansion. If $e=Wb$, shifting only the nuisance coefficient reproduces the perturbed judge logits exactly and leaves the trusted probabilities unchanged, establishing exact cancellation.
\end{proof}

Proposition~\ref{prop:app-second-order-root} explains why first-order cancellation alone is insufficient at the $n^{-1/4}$ scale. Along a sequence of learned representations and designs, a sufficient generic condition for representation error to be negligible at the oracle root-$n$ scale is
\[
\sqrt n\,\|B_me_m\|\to0,\qquad \|e_m\|=o(n^{-1/4}),
\]
together with the corresponding moment control when $e_m$ is random. Errors lying exactly in the learned nuisance span are different because they are absorbed by recentering the nuisance coefficient and are not subject to this generic first-order remainder condition.

\section{Synthetic Experiments}
\label{app-controlled-experiments}

This appendix provides the constructions and diagnostics underlying the synthetic results in Section~\ref{sec-controlled-experiments}. Section~\ref{app-controlled-protocol} specifies the common protocol, Section~\ref{app-controlled-mechanisms} gives the complete numerical evidence behind Figure~\ref{fig-controlled}, and Section~\ref{app-controlled-additional} tests mechanism boundaries not needed in the main text.

\subsection{Experimental Protocol}
\label{app-controlled-protocol}

\paragraph{Common experimental setup.}
Unless an intervention explicitly changes one component, all controlled studies use independent Bernoulli observations from the canonical joint-logit model in \eqref{eq-joint-judge-model}, with the same target and nuisance representations, parameter region, and final estimator across acquisition rules. Oracle studies condition on a fixed nuisance representation. Representation-learning studies first generate independent historical human--judge data, freeze the learned representation and acquisition rule, and then generate independent current trusted and judge observations. Thus acquisition comparisons change the information collected rather than the outcome-generating law or final estimation procedure.

\paragraph{Estimator and numerical implementation.}
The main studies use the projected one-step estimator of Appendix~\ref{app-minimax}. Trusted logistic regression supplies the target preliminary, an offset nuisance fit supplies the nuisance preliminary, and one guarded joint update is followed by projection. Numerical search and safeguard settings are summarized in Table~\ref{tab:controlled-protocol}. Numerical damping is used only to stabilize optimization and is never counted as statistical information. Fixed-support theoretical criteria are evaluated using the realized integer counts after rounding. The distinct-comparison experiment instead uses the finite-pool procedure of Algorithm~\ref{alg-naod} and the certificate in \eqref{eq-selection-certificate}.

\begin{table}[htbp]
\centering
\small
\setlength{\tabcolsep}{5pt}
\renewcommand{\arraystretch}{1.15}
\caption{\textbf{Common controlled-study protocol.} Experiment-specific constructions and sample sizes are given in Sections~\ref{app-controlled-mechanisms}--\ref{app-controlled-additional}.}
\label{tab:controlled-protocol}
\begin{tabular}{@{}p{0.31\linewidth}p{0.64\linewidth}@{}}
\hline
Component & Convention \\
\hline
Current outcomes & Independent Bernoulli trusted and judge labels \\
Final estimator & Trusted preliminary + offset nuisance preliminary + one joint update \\
Target / nuisance radii & \(2\) and \(3\), respectively, unless stated otherwise \\
Preliminary gap tolerance & \(10^{-6}\) divided by fit sample size \\
Nuisance-Gram threshold & \(10^{-10}\) \\
Joint-information guard & \(10^{-5}\) \\
Numerical damping & \(10^{-12}I\), search only \\
Final likelihood & Unpenalized \\
Fixed-support design solver & SLSQP, analytic gradient, tolerance \(10^{-12}\), cap \(600\) \\
Finite-pool procedure & Algorithm~\ref{alg-naod} framework, iteration cap \(180\) \\
Primary endpoint & Scaled policy loss \(nD_F\), unless stated otherwise \\
\hline
\end{tabular}
\end{table}

\paragraph{Replication units and uncertainty.}
For fixed-support studies, one independently generated Bernoulli dataset is one Monte Carlo replication, and MC SE is the sample standard deviation across replications divided by the square root of their number. In nested representation-learning studies, current losses are first averaged within each independent historical realization and uncertainty is computed across those historical realizations. In the distinct-pool study, current losses are first averaged within each independently generated pool and uncertainty is computed across pools. Error bars in the appendix figures are \(1.96\) MC SE. Inner current draws are therefore not treated as additional independent outer replications.

\subsection{Controlled Mechanism Studies}
\label{app-controlled-mechanisms}

\paragraph{Effective information and policy risk.}
The experiment behind Figure~\ref{fig-controlled}(a) uses target features \(X=(3,2,-1)\), constant nuisance feature \(W=1\), target center \(\theta_0=0\), nuisance center \(a_0=\log(3/2)\), and \(n_c=n\) trusted labels with feature one. The judge probability is \(0.6\) at every comparison type. The evaluation policy has feature gap two and temperature one, so \(F(\theta)=\log(1+e^{2\theta})\) up to an additive constant and \(G_0=1\). Every allocation has coordinate floor \(0.05\). NAOD, Target-info, and Random share the same correctly specified joint model and projected one-step estimator and differ only in acquisition. Each method--budget pair uses \(6{,}000\) independent datasets, and \(\Phi_\kappa\) is evaluated at the realized integer allocation.

\begin{table}[htbp]
\centering
\small
\setlength{\tabcolsep}{6pt}
\renewcommand{\arraystretch}{1.12}
\caption{\textbf{Effective-information experiment underlying Figure~\ref{fig-controlled}(a).} Each row uses \(6{,}000\) independent repetitions. \(\Phi\) is evaluated using the realized integer counts.}
\label{tab:controlled-effective}
\begin{tabular*}{\linewidth}{@{\extracolsep{\fill}}c l c c c}
\hline
\(n\) & Design & \(\Phi\) & Mean \(nD_F\) & MC SE \\
\hline
240   & NAOD        & 0.426 & 0.430 & 0.008 \\
240   & Target-info & 1.139 & 1.164 & 0.021 \\
240   & Random      & 0.530 & 0.527 & 0.010 \\
960   & NAOD        & 0.426 & 0.441 & 0.008 \\
960   & Target-info & 1.139 & 1.161 & 0.022 \\
960   & Random      & 0.530 & 0.518 & 0.010 \\
3840  & NAOD        & 0.426 & 0.426 & 0.008 \\
3840  & Target-info & 1.139 & 1.105 & 0.020 \\
3840  & Random      & 0.530 & 0.526 & 0.009 \\
15360 & NAOD        & 0.426 & 0.431 & 0.008 \\
15360 & Target-info & 1.139 & 1.144 & 0.021 \\
15360 & Random      & 0.530 & 0.541 & 0.010 \\
\hline
\end{tabular*}
\end{table}

Table~\ref{tab:controlled-effective} shows that empirical risks track the information constants across all four budgets. Because the likelihood and final estimator are identical across methods, the Target-info gap isolates acquisition geometry. Raw target-score variation is valuable only to the extent that it remains distinguishable from nuisance variation. This variance mechanism is separate from the exposure effect in \eqref{eq-residual-exposure}. Under reference weights \((0.1,0.2,0.7)\), the probability-residual score balances to zero, whereas uniform exposure gives \(b_\xi=2/15\).

\paragraph{Finite-sample representation-learning cost.}
The experiment behind Figure~\ref{fig-controlled}(b) uses two equally weighted comparison types with \(X=(1,-1)^\top\), true nuisance direction \(W_0=(1,1)^\top\), \(\theta_0=0\), and \(a_0=c=\log(3/2)\). Human probabilities are \(0.5\) and judge probabilities are \(0.6\). The historical sample contains \(m\) independent units, each contributing one human and one judge label, with the units split equally across the two comparison types. Let \(\widehat h_i\) and \(\widehat j_i\) denote the corresponding sample means, let \(s=1/4\), \(t=0.24\), and set \(g=1/2\). The historical learner first estimates the rotation
\[
\delta_m^{\rm raw}
=
\frac{g\{(\widehat j_1-\widehat h_1)-(\widehat j_2-\widehat h_2)\}}{2ct}.
\]
We then set \(\widehat\delta_m=\operatorname{clip}(\delta_m^{\rm raw},-0.5,0.5)\) and \(\widehat W_m=(1+\widehat\delta_m,1-\widehat\delta_m)^\top\).
The current design is equally weighted with \(n_c=n\). Propagating the historical rotation through the target-influence operator gives \(\Phi_0=50/49\) and \(C_P=25/98\), so \eqref{eq-representation-risk} predicts scaled risk \(\Phi_0+(n/m)C_P\). Each setting below uses \(6{,}000\) independent historical--current replications.

\begin{table}[htbp]
\centering
\small
\setlength{\tabcolsep}{5pt}
\renewcommand{\arraystretch}{1.12}
\caption{\textbf{Representation-learning cost underlying Figure~\ref{fig-controlled}(b).} Prediction is \(\Phi_0+(n/m)C_P\). ``Clip'' is the percentage of historical fits for which the bounded rotation is active.}
\label{tab:controlled-representation-cost}
\begin{tabular*}{\linewidth}{@{\extracolsep{\fill}}c c c c c c}
\hline
\(n\) & \(m/n\) & Prediction & Mean \(nD_F\) & MC SE & Clip (\%) \\
\hline
960   & \(1/4\) & 2.041 & 1.891 & 0.034 & 3.1 \\
960   & \(1\)   & 1.276 & 1.252 & 0.023 & 0.0 \\
960   & \(4\)   & 1.084 & 1.085 & 0.020 & 0.0 \\
3840  & \(1/4\) & 2.041 & 1.993 & 0.037 & 0.0 \\
3840  & \(1\)   & 1.276 & 1.281 & 0.024 & 0.0 \\
3840  & \(4\)   & 1.084 & 1.116 & 0.020 & 0.0 \\
15360 & \(1/4\) & 2.041 & 2.046 & 0.037 & 0.0 \\
15360 & \(1\)   & 1.276 & 1.299 & 0.024 & 0.0 \\
15360 & \(4\)   & 1.084 & 1.091 & 0.020 & 0.0 \\
\hline
\end{tabular*}
\end{table}

At fixed \(m/n\), the predicted scaled risk is asymptotically constant rather than vanishing with the absolute current sample size. The empirical values in Table~\ref{tab:controlled-representation-cost} approach these common levels as both samples grow, confirming that representation learning contributes at the same leading order as current-sample variance when \(n/m\) does not vanish. The largest discrepancy occurs in the only setting with nonzero clipping, identifying a finite-sample departure from the local regime rather than a different asymptotic mechanism.

\paragraph{Design-ranking reversal.}
The experiment behind Figure~\ref{fig-controlled}(c) returns to the three-type persistent-residual construction with \(n=n_c=960\). For each \(m/n\in\{1/4,1/2,1,2\}\), we generate \(300\) independent balanced historical datasets, each containing \(m\) independent units with one human and one judge label per unit. At each comparison type, clipped human and judge sample means are transformed to logits and differenced. The resulting discrepancy vector is normalized to form the learned one-dimensional nuisance representation. A historical trusted logistic fit supplies the target center and an offset logistic fit supplies the nuisance coefficient. NAOD and Target-info use this same learned representation, current budgets, and final estimator.

Each historical realization generates eight current datasets per design. Losses are first averaged within the historical realization, and paired uncertainty is computed across the \(300\) outer averages. Table~\ref{tab:controlled-reversal} reports the NAOD-minus-Target-info scaled-risk difference. Positive values favor Target-info.

\begin{table}[htbp]
\centering
\small
\setlength{\tabcolsep}{8pt}
\renewcommand{\arraystretch}{1.12}
\caption{\textbf{Design-ranking reversal underlying Figure~\ref{fig-controlled}(c).} Differences are paired over \(300\) independent historical realizations, each averaging eight current datasets. Intervals are pointwise two-sided \(t_{299}\) intervals.}
\label{tab:controlled-reversal}
\begin{tabular*}{\linewidth}{@{\extracolsep{\fill}}c c c c}
\hline
\(m/n\) & NAOD--Target-info & Outer SE & 95\% interval \\
\hline
\(1/4\) & \(+1.588\) & 0.237 & \([1.122,\,2.054]\) \\
\(1/2\) & \(+0.772\) & 0.200 & \([0.378,\,1.165]\) \\
\(1\)   & \(-0.039\) & 0.100 & \([-0.235,\,0.157]\) \\
\(2\)   & \(-0.446\) & 0.049 & \([-0.541,\,-0.350]\) \\
\hline
\end{tabular*}
\end{table}

The sign change is the empirical counterpart of \eqref{eq-design-reversal}. Representation sensitivity dominates when historical information is scarce, the two leading contributions nearly balance around \(m/n=1\), and the oracle acquisition advantage dominates once representation error is sufficiently reduced. Because both designs use the same learned nuisance model and final estimator, the reversal is induced by how acquisition responds to representation error rather than by a difference in fitted model class.

\subsection{Additional Ablations and Diagnostics}
\label{app-controlled-additional}

\paragraph{The direction of omitted error matters.}
The representation analysis in Appendix~\ref{app-representation-error} predicts that policy cost depends on how omitted logit error passes through the target-influence operator, not on its norm alone. At the zero-center NAOD allocation, let \(u_1\) be the unit direction aligned with the target-influence row. A near-null direction \(u_2\) is obtained by projecting the all-ones vector onto the orthogonal complement of \(u_1\) and rotating it by \(\pi/12\) toward \(u_1\). At \(n=3840\), both directions receive the same magnitudes \(\rho\in\{0,1,2,3,4,6,8\}\), with true judge logits perturbed by \(\rho u_j/\sqrt n\). Each point uses \(6{,}000\) independent datasets.

\begin{figure}[htbp]
    \centering
    \includegraphics[width=\linewidth]{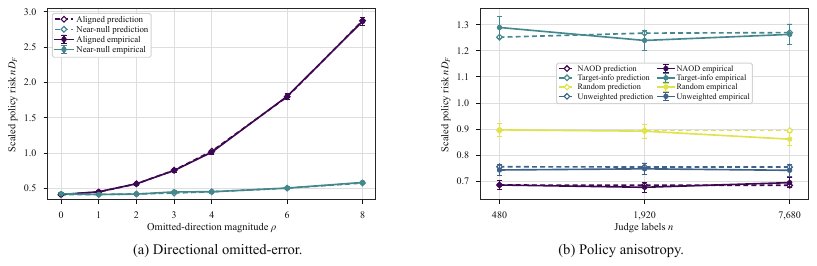}
    \vspace{-7mm}
    \caption{
        \textbf{Controlled diagnostics.}
        \textbf{(a) Directional omitted-error diagnostic.}
        Dashed curves are local predictions from the frozen logit-influence kernel. Solid curves are empirical mean scaled policy risks.
        The two perturbation directions have identical Euclidean magnitudes.
        \textbf{(b) Policy-anisotropy diagnostic.}
        Dashed curves are information-based predictions. Solid curves are empirical mean scaled policy risks.
        The Unweighted rule changes only the acquisition weight and is evaluated under the same policy loss.
        Error bars show \(1.96\) MC SE over \(6{,}000\) independent datasets.
    }
    \label{fig:controlled-diagnostics}
\end{figure}

Figure~\ref{fig:controlled-diagnostics}(a) shows that equal error magnitude can produce sharply different policy costs. The aligned perturbation follows the predicted quadratic increase, whereas the near-null perturbation remains much less consequential. Residual magnitude alone therefore cannot rank acquisition risk. The relevant quantity is its policy-weighted target projection through the target-influence operator and \(G_0\).

\paragraph{Policy weighting changes which information is useful.}
To isolate \(G_0\), we use target rows \((3,0),(2,0),(-1,0),(0,3),(0,2),(0,-1)\), constant nuisance feature \(W=1\), \(\theta_0=0\), and \(a_0=\log(3/2)\). Half of the trusted observations use each coordinate direction, giving \(H_c=I_2/8\). The evaluation distribution weights two binary-action contexts by \(0.8\) and \(0.2\), each with feature gap two, so
\[
F(\theta)=0.8\log(1+e^{2\theta_1})+0.2\log(1+e^{2\theta_2}),\qquad G_0=\operatorname{diag}(0.8,0.2).
\]
The Unweighted ablation replaces this policy curvature by \(I_2/2\) only during acquisition and is evaluated under the same true policy loss. The allocation floor is \(0.03\), \(n_c=n\), and each setting uses \(6{,}000\) repetitions.

Figure~\ref{fig:controlled-diagnostics}(b) shows that NAOD remains below the Unweighted design across all budgets, even though both account for nuisance estimation. The gap therefore isolates policy weighting. \(I_{\mathrm{eff}}\) determines which target directions remain identifiable after nuisance adjustment, while \(G_0\) determines which of those directions matter for downstream policy loss. Target-info incurs the additional cost of ignoring target--nuisance coupling.

\paragraph{Nuisance rank separates misspecification from information loss.}
Using the same two-dimensional construction at \(n=3840\), we vary only the fitted nuisance rank while keeping the actual judge probability at \(0.6\). Underfit uses no nuisance feature, Correct uses the constant feature, Overfit adds an unnecessary axis-group contrast, and Saturated uses \(W=I_6\). In the saturated model, every judge target-score direction lies in the nuisance span.

\begin{table}[htbp]
\centering
\small
\setlength{\tabcolsep}{7pt}
\renewcommand{\arraystretch}{1.12}
\caption{\textbf{Nuisance-rank ablation at \(n=3840\).} The Underfit nominal criterion is not a true-risk prediction because its fitted model omits the actual judge discrepancy.}
\label{tab:controlled-rank}
\begin{tabular*}{\linewidth}{@{\extracolsep{\fill}}l c c c c}
\hline
Fitted model & \(r\) & Nominal \(\Phi\) & Mean \(nD_F\) & MC SE \\
\hline
Underfit  & 0 & 0.391 & 27.562 & 0.064 \\
Correct   & 1 & 0.686 & 0.699  & 0.010 \\
Overfit   & 2 & 0.772 & 0.764  & 0.011 \\
Saturated & 6 & 4.000 & 4.015  & 0.059 \\
\hline
\end{tabular*}
\end{table}

Table~\ref{tab:controlled-rank} separates three distinct effects. Underfitting creates omitted-discrepancy bias, so its favorable nominal criterion is misleading. Overfitting remains correctly specified but incurs a variance cost from an unnecessary nuisance direction. Saturation is also correctly specified, but the Schur complement removes all judge-derived target information because nuisance variation can reproduce every judge target-score direction. Nuisance dimension is therefore neither uniformly beneficial nor uniformly harmful. Its effect depends on representation adequacy and target--nuisance geometry.

\paragraph{Distinct finite-pool selection.}
We generate \(24\) independent pools, each containing \(N=600\) target-feature rows sampled uniformly from \([-2,2]^2\). We set \(W_i=1\), \(\theta_0=0\), \(a_0=\log(3/2)\), and use \(300\) trusted labels with balanced coordinate features and the same \(0.8/0.2\) policy weighting as above. Judge budgets are \(B\in\{60,180,360\}\). NAOD and Target-info use the finite-pool relaxation, rounding, and exchange procedure of Algorithm~\ref{alg-naod} with their respective frozen objectives and a common charged seed, while Random samples the remaining IDs without replacement. For each pool and budget, \(300\) independent outcome datasets are generated. Only selected outcomes enter the fit, and uncertainty is computed across the \(24\) pool-level mean losses.

\begin{table}[htbp]
\centering
\small
\setlength{\tabcolsep}{6pt}
\renewcommand{\arraystretch}{1.12}
\caption{\textbf{Distinct finite-pool experiment.} Risks are reported in \(10^{-3}\) units. Nominal risk uses the actual selected-subset information. SE is computed across \(24\) independent pool-level means, each averaging \(300\) current datasets.}
\label{tab:controlled-finite-pool}
\begin{tabular*}{\linewidth}{@{\extracolsep{\fill}}c l c c c}
\hline
\(B\) & Design & Nominal risk & Mean \(D_F\) & Pool SE \\
& & \((\times10^{-3})\) & \((\times10^{-3})\) & \((\times10^{-3})\) \\
\hline
60  & NAOD        & 6.108 & 6.146 & 0.081 \\
60  & Target-info & 6.145 & 6.181 & 0.081 \\
60  & Random      & 8.945 & 8.851 & 0.117 \\
180 & NAOD        & 3.420 & 3.361 & 0.053 \\
180 & Target-info & 3.430 & 3.373 & 0.051 \\
180 & Random      & 5.287 & 5.191 & 0.082 \\
360 & NAOD        & 2.459 & 2.376 & 0.029 \\
360 & Target-info & 2.461 & 2.385 & 0.030 \\
360 & Random      & 3.293 & 3.224 & 0.046 \\
\hline
\end{tabular*}
\end{table}

Table~\ref{tab:controlled-finite-pool} shows that the realized-subset criterion tracks empirical policy loss without requiring the rounded subset to approximate a fixed-support fractional allocation. Across the \(24\) pools, the largest NAOD Frank--Wolfe gap is \(4.0\times10^{-8}\) and the largest integer certificate is \(1.1\times10^{-6}\). The corresponding Target-info maxima are \(5.7\times10^{-9}\) and \(1.2\times10^{-7}\). These values certify numerical optimization of each frozen criterion, not statistical uncertainty.

As a numerical check, we also compared the projected one-step estimator with a fully converged constrained joint likelihood on paired datasets. Their scaled-risk differences decreased with sample size and were negligible relative to the acquisition gaps.

\section{Chatbot Arena Evaluation}
\label{app-arena}

\subsection{Data, Representations, and Experimental Protocol}
\label{app-arena-protocol}

\paragraph{Frozen archive and cluster construction.}
We use the frozen revision of the Chatbot Arena LLM-judge archive. The raw archive contains \(49{,}938\) rows. After removing non-string or empty text and identical-response comparisons, \(49{,}635\) usable rows remain in \(42{,}875\) connected text clusters. Text is Unicode NFKC normalized and whitespace is collapsed. Connected components are formed before filtering by transitive exact-text matching on the prompt or either response. Cluster identity is the ownership unit for splitting and the capacity unit for acquisition.

\begin{table}[htbp]
\centering
\small
\setlength{\tabcolsep}{5pt}
\renewcommand{\arraystretch}{1.12}
\caption{\textbf{Arena cluster partition within each repeated split.} Roles are cluster-disjoint. Candidate clusters may contain multiple rows, but at most one row from a connected cluster can be acquired.}
\label{tab:arena-roles}
\begin{tabular}{@{}p{0.29\linewidth}c p{0.49\linewidth}@{}}
\hline
Role & Clusters & Use \\
\hline
Upstream & \(8{,}575\) & Target representation, nuisance representation, and human reference \\
Historical initialization & \(128\) & Initial target and nuisance estimates \\
Policy support & \(1{,}024\) & Policy weighting covariates with no labels \\
Current human reservoir & \(512\) & Nested trusted budgets \(H\leq512\) \\
Candidate & \(24{,}061\) & Judge acquisition with cluster capacity one \\
Test & \(8{,}575\) & Held-out human evaluation \\
\hline
\end{tabular}
\end{table}

Table~\ref{tab:arena-roles} summarizes these role assignments. Upstream, historical initialization, policy support, and the human reservoir use one representative row from each assigned cluster. Candidate retains all eligible rows subject to cluster capacity one. Test performance is first averaged within connected clusters. We use \(15\) prespecified repeated cluster-level splits. Each split relearns the downstream representations and human reference, while all \(15\) budget combinations within a split share these upstream objects.

\paragraph{Target and nuisance representations.}
Response features are extracted with Qwen3-Embedding-0.6B. The pair representation is the difference between the A and B response embeddings. An upstream-fitted whitened PCA reduces this representation to \(128\) coordinates.

For each judge, we fit four ridge heads separately to human preferences and to that judge's hard preferences. The leading two-dimensional singular-vector span of the resulting coefficient vectors defines the target representation \(X_J\). Policy-weighted residual learning is performed separately. The resulting standardized swap-antisymmetric residual score \(r_J(z)\) defines the nuisance representation \(W_J(z)=(1,r_J(z))^\top\). Thus the target and nuisance representations are learned from upstream data and frozen before current acquisition. Table~\ref{tab:arena-config} records the frozen numerical configuration used across methods.

\begin{table}[htbp]
\centering
\small
\setlength{\tabcolsep}{5pt}
\renewcommand{\arraystretch}{1.12}
\caption{\textbf{Frozen Arena configuration.} All acquisition methods use the same representation, initialization, current budgets, and final estimator unless required otherwise by the method definition.}
\label{tab:arena-config}
\begin{tabular}{@{}p{0.32\linewidth}p{0.63\linewidth}@{}}
\hline
Component & Configuration \\
\hline
Embedding & Qwen3-Embedding-0.6B, last-token pooling, \(L_2\) normalization \\
Maximum input & \(2{,}048\) tokens with prompt cap \(512\) \\
Base reduction & Upstream-fitted whitened PCA, \(128\) dimensions \\
Target representation & Judge-specific SVD span, dimension \(d=2\) \\
Target-head penalties & \(32,128,512,2048\) \\
Nuisance representation & Intercept and standardized learned residual, dimension \(r=2\) \\
Residual learner & ExtraTrees, \(128\) trees, depth \(10\), min leaf \(20\), feature fraction \(0.5\), weight clip \([0.1,10]\) \\
Inner residual folds & \(4\) \\
Human budgets & \(H\in\{32,64,128,256,512\}\) \\
Judge budgets & \(B\in\{256,512,1024\}\) \\
Parameter boxes & \(\|\theta\|_\infty\leq20\), \(\|a\|_\infty\leq10\) \\
Final step size & \(\alpha=1\) \\
Finite-pool solver & Frank--Wolfe cap \(180\), tolerance \(10^{-5}\) \\
Cluster capacity & At most one selected comparison per connected cluster \\
\hline
\end{tabular}
\end{table}

\paragraph{Feedback and common estimator.}
Human outcomes are encoded as \(1\) for A, \(0\) for B, and \(1/2\) for ties. For judge feedback, the archived A/B/tie probabilities are converted to the soft target \(y_J=p_A+\tfrac12 p_T\). The logistic quasi-likelihood therefore uses the archived graded preference signal directly.

A common historical set of \(128\) comparisons supplies the initial target and nuisance estimates before current acquisition. After current trusted feedback and the selected judge feedback are revealed, every acquisition method uses the same unpenalized joint observed-information update and the same fixed parameter boxes. Historical initialization rows are not reused in the final score. NAOD and Target-info use the same charged feasibility anchors. Their comparison therefore changes the acquisition criterion while holding the representation, budgets, and final estimator fixed.

\paragraph{Resource accounting.}
The varying current resource is \((H,B)\). The \(8{,}575\) upstream human--judge comparisons used for representation learning and the \(128\) historical comparisons used for initialization are shared across acquisition rules and are not counted in \(H\) or \(B\). Policy support contributes covariates without labels, while test human outcomes are used only for evaluation.

\subsection{Metrics, Baselines, and Statistical Analysis}
\label{app-arena-analysis}

\paragraph{Acquisition baselines.}
Random samples feasible candidate comparisons uniformly. Entropy prioritizes candidates with high predictive uncertainty. D-opt and PA D-opt use Fisher-information acquisition on the common judge-specific target representation, with PA D-opt additionally incorporating the historical initialization as past information. Target-info minimizes the policy-weighted target-information criterion without nuisance adjustment. NAOD instead uses effective target information after nuisance adjustment. Initial performs no current update. Human-only updates from the common initialization using current trusted feedback without acquired judge outcomes.

\paragraph{Evaluation endpoints.}
The primary endpoint is proxy policy regret \(D_F(\theta_{\mathrm{ref}},\widehat\theta)\), where \(\theta_{\mathrm{ref}}\) is fitted only from upstream human feedback. We additionally report held-out human cross-entropy and tie-aware human choice accuracy. Cross-entropy uses the full predicted probability. For hard choice accuracy, we predict A when the fitted preference probability is at least \(0.5\) and B otherwise. Human ties receive half credit.

\paragraph{Aggregation and paired uncertainty.}
Test losses are first averaged within connected clusters and then equally across clusters. Judges are macro-averaged within each split. For overall results, the \(15\) budget combinations are also equally weighted. The repeated split is the inference unit. For loss endpoints, paired gain is baseline minus NAOD. For accuracy, paired gain is NAOD minus baseline. Positive values therefore favor NAOD throughout. Two-sided \(95\%\) Student \(t\) intervals are computed from the \(15\) paired split-level differences using \(t_{14}\). Budget and judge intervals are pointwise intervals. Relative proxy-regret reduction is \(100(1-\overline R_{\mathrm{NAOD}}/\overline R_b)\), which is a ratio of aggregated mean regrets rather than an average of split-specific percentage reductions.

\subsection{Complete Results and Diagnostics}
\label{app-arena-results}

\paragraph{Paired comparisons across methods.}
The main text reports absolute endpoint means. Table~\ref{tab:arena-paired-overall} reports the corresponding paired split-level effects and their uncertainty.

\begin{table}[htbp]
\centering
\small
\setlength{\tabcolsep}{5pt}
\renewcommand{\arraystretch}{1.12}
\caption{\textbf{Overall paired Arena comparisons.} Positive values favor NAOD. Proxy-regret and human-CE gains are in \(10^{-3}\) units. Accuracy gains are percentage points. Brackets give paired \(95\%\) \(t_{14}\) intervals across the \(15\) repeated splits.}
\label{tab:arena-paired-overall}
\begin{tabular*}{\linewidth}{@{\extracolsep{\fill}}l c c c}
\hline
Baseline & Proxy-regret gain & Human-CE gain & Accuracy gain \\
\hline
Initial
& \(+6.870\,[3.020,10.721]\)
& \(+5.308\,[1.952,8.664]\)
& \(+0.448\,[0.117,0.780]\) \\
Human-only
& \(+16.259\,[8.849,23.668]\)
& \(+12.981\,[7.214,18.748]\)
& \(+1.426\,[0.331,2.520]\) \\
Random
& \(+2.886\,[2.455,3.318]\)
& \(+2.814\,[2.388,3.241]\)
& \(+0.050\,[-0.069,0.169]\) \\
Entropy
& \(+11.433\,[4.783,18.083]\)
& \(+8.527\,[3.669,13.385]\)
& \(+1.166\,[0.079,2.252]\) \\
D-opt
& \(+0.463\,[0.093,0.833]\)
& \(+0.385\,[0.011,0.759]\)
& \(+0.075\,[0.028,0.121]\) \\
PA D-opt
& \(+0.465\,[0.095,0.834]\)
& \(+0.386\,[0.013,0.760]\)
& \(+0.075\,[0.029,0.121]\) \\
Target-info
& \(+0.818\,[0.348,1.287]\)
& \(+0.531\,[0.211,0.850]\)
& \(+0.058\,[0.015,0.101]\) \\
\hline
\end{tabular*}
\end{table}

Relative to Target-info, the proxy-regret gain is accompanied by lower held-out human cross-entropy and higher held-out choice accuracy. The D-opt and PA D-opt differences are smaller but remain positive in the overall paired analysis. The smaller accuracy gain is expected because hard choice accuracy changes only when the predicted preference crosses the \(0.5\) threshold, whereas cross-entropy remains sensitive to probability improvements that preserve the predicted winner. The resulting accuracy gain is positive in \(12\) of \(15\) repeated splits.

\paragraph{Budget dependence.}
Figure~\ref{fig-arena-budgets} reports relative proxy-regret reductions over the complete \(5\times3\) budget grid. Table~\ref{tab:arena-budget-paired} reports the paired absolute gains and uncertainty for Target-info and D-opt.

\begin{figure}[htbp]
\centering
\includegraphics[width=\linewidth]{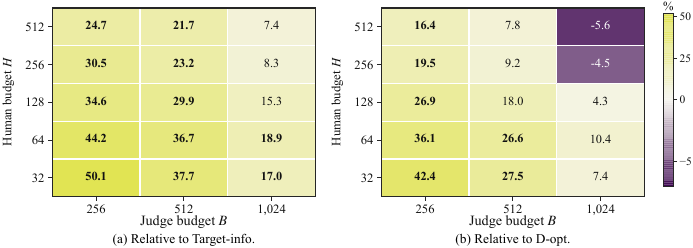}
\vspace{-7mm}
\caption{\textbf{Budget dependence on Chatbot Arena.} NAOD relative reduction in mean proxy policy regret across the complete \(5\times3\) Arena budget grid. Positive values favor NAOD. Panel (a) compares against Target-info and panel (b) against D-opt. Cell values are ratios of macro-averaged mean regrets rather than averages of split-specific percentage reductions.}
\label{fig-arena-budgets}
\end{figure}

\begin{table}[t]
\centering
\small
\setlength{\tabcolsep}{4pt}
\renewcommand{\arraystretch}{1.10}
\caption{\textbf{Paired proxy-regret gains across the Arena budget grid.} Entries are baseline-minus-NAOD gains in \(10^{-3}\) units with pointwise paired \(95\%\) \(t_{14}\) intervals.}
\label{tab:arena-budget-paired}
\begin{tabular*}{\linewidth}{@{\extracolsep{\fill}}c c c c}
\hline
\(H\) & \(B\) & Target-info--NAOD & D-opt--NAOD \\
\hline
32  & 256  & \(+2.103\,[0.263,3.944]\)  & \(+1.543\,[0.612,2.474]\) \\
32  & 512  & \(+1.141\,[0.447,1.835]\)  & \(+0.715\,[0.138,1.291]\) \\
32  & 1024 & \(+0.408\,[0.143,0.673]\)  & \(+0.160\,[-0.192,0.512]\) \\
64  & 256  & \(+1.731\,[0.507,2.955]\)  & \(+1.234\,[0.359,2.109]\) \\
64  & 512  & \(+1.123\,[0.476,1.770]\)  & \(+0.703\,[0.125,1.281]\) \\
64  & 1024 & \(+0.471\,[0.184,0.758]\)  & \(+0.234\,[-0.119,0.587]\) \\
128 & 256  & \(+1.172\,[0.205,2.138]\)  & \(+0.814\,[0.062,1.565]\) \\
128 & 512  & \(+0.843\,[0.071,1.615]\)  & \(+0.434\,[-0.117,0.986]\) \\
128 & 1024 & \(+0.367\,[-0.044,0.778]\) & \(+0.092\,[-0.303,0.487]\) \\
256 & 256  & \(+0.844\,[0.211,1.477]\)  & \(+0.466\,[0.008,0.924]\) \\
256 & 512  & \(+0.553\,[0.221,0.885]\)  & \(+0.186\,[-0.222,0.594]\) \\
256 & 1024 & \(+0.176\,[-0.044,0.396]\) & \(-0.084\,[-0.405,0.238]\) \\
512 & 256  & \(+0.664\,[0.203,1.124]\)  & \(+0.395\,[0.001,0.790]\) \\
512 & 512  & \(+0.515\,[0.199,0.831]\)  & \(+0.157\,[-0.186,0.500]\) \\
512 & 1024 & \(+0.155\,[-0.044,0.354]\) & \(-0.103\,[-0.369,0.162]\) \\
\hline
\end{tabular*}
\end{table}

The Target-info mean comparison favors NAOD in all \(15\) budget cells, with \(12\) pointwise paired intervals entirely above zero. At every fixed human budget, the relative gain narrows as the judge budget increases. The pattern is not determined by \(H/B\) alone. Along \(H/B=1/8\), the relative reductions are \(50.1\%\), \(36.7\%\), and \(15.3\%\) for \((H,B)=(32,256),(64,512),(128,1024)\). Selection overlap increases from \(0.446\) to \(0.505\) to \(0.568\) over the same sequence. Absolute information scale and finite-pool geometry therefore matter in addition to the human-to-judge budget ratio.

The D-opt gap narrows more sharply at large judge budgets. At \((256,1024)\) and \((512,1024)\), the mean difference changes sign, while both paired intervals include zero. The separation from D-opt is therefore budget dependent rather than uniform across the grid.

\paragraph{Judge-level target--nuisance coupling.}
Figure~\ref{fig-arena-coupling} summarizes the judge-level association between Target-info coupling and realized policy gain, while Table~\ref{tab:arena-judge} reports the corresponding coupling values and paired policy gains with pointwise \(95\%\) intervals.

\begin{figure}[htbp]
\centering
\includegraphics[width=\linewidth]{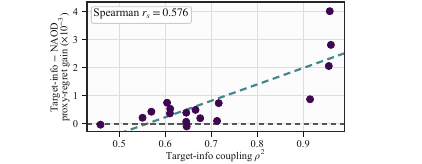}
\vspace{-7mm}
\caption{\textbf{Judge-level coupling and realized acquisition gain.} The horizontal axis is Target-info target--nuisance coupling \(\rho^2\) and the vertical axis is the Target-info-minus-NAOD proxy-regret gain in \(10^{-3}\) units. Each point represents one judge after averaging over the \(15\) budget combinations and \(15\) repeated cluster-level splits. The Spearman rank correlation is \(r_s=0.576\). Positive vertical values favor NAOD.}
\label{fig-arena-coupling}
\end{figure}

\begin{table}[t]
\centering
\small
\setlength{\tabcolsep}{4pt}
\renewcommand{\arraystretch}{1.08}
\caption{\textbf{Judge-level coupling and policy gain.} \(\rho_T^2\) is the squared target--nuisance coupling under Target-info. Policy gain is Target-info minus NAOD in \(10^{-3}\) units with pointwise paired \(95\%\) \(t_{14}\) intervals.}
\label{tab:arena-judge}
\begin{tabular*}{\linewidth}{@{\extracolsep{\fill}}l c c}
\hline
Judge & \(\rho_T^2\) & Policy gain \\
\hline
Athene-70B & 0.570 & \(+0.428\,[-0.290,1.145]\) \\
dolphin-2.1-mistral-7b & 0.956 & \(+2.063\,[1.076,3.049]\) \\
dolphin-2.5-mixtral-8x7b & 0.960 & \(+2.819\,[1.248,4.390]\) \\
Hermes-3-Llama-3.1-70B & 0.610 & \(+0.366\,[-0.305,1.038]\) \\
Meta-Llama-3-70B-Instruct & 0.459 & \(-0.037\,[-0.358,0.285]\) \\
Meta-Llama-3-8B-Instruct & 0.716 & \(+0.733\,[0.069,1.397]\) \\
Meta-Llama-3-8B & 0.958 & \(+4.024\,[1.184,6.865]\) \\
Mistral-7B-Instruct-v0.1 & 0.713 & \(+0.096\,[-0.035,0.227]\) \\
Mistral-7B-Instruct-v0.2 & 0.666 & \(+0.491\,[0.129,0.854]\) \\
Mistral-7B-OpenOrca & 0.646 & \(+0.067\,[-0.151,0.285]\) \\
Mixtral-8x7B-Instruct-v0.1 & 0.604 & \(+0.749\,[0.332,1.166]\) \\
OpenHermes-2-Mistral-7B & 0.676 & \(+0.194\,[-0.141,0.529]\) \\
OpenHermes-2.5-Mistral-7B & 0.611 & \(+0.531\,[0.153,0.909]\) \\
Qwen2-72B-Instruct & 0.551 & \(+0.211\,[-0.252,0.673]\) \\
StableBeluga-7B & 0.915 & \(+0.873\,[0.234,1.511]\) \\
Starling-LM-7B-alpha & 0.646 & \(+0.391\,[0.045,0.736]\) \\
zephyr-7b-beta & 0.647 & \(-0.099\,[-0.283,0.084]\) \\
\hline
\end{tabular*}
\end{table}

Mean policy gain is positive for \(15\) of the \(17\) judges. The largest gains occur for Meta-Llama-3-8B and the two dolphin judges, whose Target-info coupling values all exceed \(0.95\). By contrast, Meta-Llama-3-70B-Instruct has strong residual predictability but substantially weaker coupling and essentially no policy gain. This contrast separates residual learnability from target--nuisance interference. The Spearman association in Figure~\ref{fig-arena-coupling} summarizes this judge-level relationship.

Across all budget and split designs, mean squared coupling decreases from \(0.700\) under Target-info to \(0.290\) under NAOD. The mean selected-set overlap is \(0.521\). NAOD therefore changes which comparisons are acquired in a direction that reduces target--nuisance coupling rather than applying a different estimator to essentially the same selected set.

\paragraph{Representation and numerical diagnostics.}
All \(255=15\times17\) split-by-judge nuisance fits achieve positive out-of-fold soft-label cross-entropy gain, with mean gain \(0.0370\). The learned nuisance representation therefore captures systematic judge residual structure across the full evaluation.

The audit finds no budget, duplicate-ID, duplicate-cluster, out-of-pool, projection, or design failures in \(22{,}950\) acquisition arrays. Across \(7{,}650\) NAOD and Target-info designs, the maximum Frank--Wolfe gap is \(8.62\times10^{-8}\) and the maximum integer certificate is \(6.01\times10^{-7}\), indicating that the observed acquisition differences are not explained by unresolved selection optimization.

\end{document}